\documentclass{article}

\usepackage{microtype}
\usepackage{graphicx}
\usepackage{subfigure}
\usepackage{booktabs} 
\usepackage{xcolor}
\usepackage{bbm}
\usepackage{mathtools}
\usepackage{natbib}
\usepackage{verbatim}
\usepackage[linesnumbered, ruled,vlined]{algorithm2e}
\usepackage{algorithm,algorithmic}
\usepackage{enumitem}
\setlist[itemize]{topsep=0pt, leftmargin=3mm}

\usepackage{microtype}
\usepackage{graphicx}
\usepackage{subfigure}
\usepackage{booktabs} 
\usepackage{hyperref}
\usepackage{wrapfig}
\usepackage{graphicx}

\usepackage{tikz}
\usetikzlibrary{positioning}
\usetikzlibrary{intersections,shapes.arrows}

\usepackage{amsthm,amsmath,amssymb, amsfonts}
\newtheorem{theorem}{Theorem}[section]
\newtheorem{lemma}[theorem]{Lemma}

\allowdisplaybreaks

\usepackage{amsmath}

\usepackage{hyperref}

\usepackage[preprint, nonatbib]{neurips_2019}

\usepackage{comment}

\title{Learning to Score Behaviors for Guided Policy Optimization}

\begin{document}

\author{%
  Aldo Pacchiano$^*$ \\
  UC Berkeley \\
   \And
  Jack Parker-Holder$^*$ \\
  University of Oxford \\
   \And
  Yunhao Tang$^*$ \\
  Columbia University \\
   \AND
  Anna Choromanska \\
  NYU \\
  \And
  Krzysztof Choromanski\\
  Google Brain Robotics\\
   \And
  Michael I. Jordan \\
  UC Berkeley \\
   \\[-50.0ex]
   }

\maketitle

\vskip 0.3in

\def\thefootnote{*}\footnotetext{Equal contribution.}

\begin{abstract}
We introduce a new approach for comparing reinforcement learning policies, using Wasserstein distances (WDs) in a newly defined latent behavioral space. We show that by utilizing the dual formulation of the WD, we can learn score functions over policy behaviors that can in turn be used to lead policy optimization towards (or away from) (un)desired behaviors. Combined with smoothed WDs, the dual formulation allows us to devise efficient algorithms that take stochastic gradient descent steps through WD regularizers. We incorporate these regularizers into two novel on-policy algorithms, Behavior-Guided Policy Gradient and Behavior-Guided Evolution Strategies, which we demonstrate can outperform existing methods in a variety of challenging environments. We also provide an open source demo\footnote{Available at \url{https://github.com/behaviorguidedRL/BGRL}. We emphasize this is the exact code from our experiments, but a demo to build intuition and clarify our methods.}.
\end{abstract}

\section{Introduction}
One of the key challenges in reinforcement learning (RL) is to efficiently incorporate the behaviors of learned policies into optimization algorithms \citep{b1, meyerson, conti}. The fundamental question we aim to shed light on in this paper is:

\begin{center}
\textit{What is the right measure of similarity between two policies acting on the same underlying MDP and how can we devise algorithms to leverage this information for RL?}
\end{center}
In simple terms, the main thesis motivating the methods we propose is that: 
\begin{center}
\textit{Two policies may perform similar actions at a local level but result in very different global behaviors.}
\end{center}

We propose to define \textit{behaviors} via so-called Behavioral Policy Embeddings (henceforth referred to as Policy Embeddings), which can be both on policy and off policy.

On policy embeddings are achieved via what we call Behavioral Embeddings Maps (BEMs) - functions mapping trajectories of a policy into a latent behavioral space representing trajectories in a compact way. We define the policy embedding as the pushforward distributions over trajectory embeddings as a result of applying a BEM to the policy's  trajectories. Importantly, two policies with distinct distributions over trajectories may result in the same probabilistic embedding. Off policy embeddings in contrast correspond to state and policy evaluation pairs resulting of evaluating the policy on states sampled from a probing state distribution that can be chosen independently from the policy. 

Both embedding mechanisms result in probabilistic Policy Embeddings, which allow us to identify a policy with a distribution with support on an embedding space. Policy Embeddings provide us a way to rigorously define dissimilarity between policies. We do this by equipping them with metrics defined on the manifold of probability measures, namely a class of Wasserstein distances (WDs, \cite{villani}). There are several reasons for choosing WDs:



\begin{itemize}
    \item \textbf{Flexibility}. We can use any cost function between embeddings of trajectories, allowing the distance between policy embeddings to arise organically from an interpretable distance between embedding points.  
    \item \textbf{Non-injective BEMs}. Different trajectories may be mapped to the same embedding point (for example in the case of the last-state embedding). This precludes the use of likelihood-based distances such as the KL divergence \citep{kullback1951}, which we discuss in Section \ref{sec:related}. 
    \item \textbf{Behavioral Test Functions}. Solving the dual formulation of the WD objective yields a pair of test functions over the space of embeddings, used to score trajectories or state policy pairs
    (see: Sec. \ref{section::repulsion_learning}).
\end{itemize}

The Behavioral Test Functions, underpin all our algorithms, directing optimization towards desired behaviors. To learn them, it suffices to define the embedding type and BEM (if required) and the cost function between points in the resulting behavioral manifold. To mitigate the computational burden of computing WDs, we rely on their entropy-regularized formulations. This allows us to update the learned test functions in a computationally efficient manner via stochastic gradient descent (SGD) on a Reproducing Kernel Hilbert Space (RKHS). We develop a novel method for stochastic optimal transport based on random feature maps \citep{randomfeatures2007} to produce compact and memory-efficient representations of learned behavioral test functions. Finally, having laid the groundwork for comparing policies via behavior-driven trajectory or state-policy pairs scores, we address our core question by introducing two new on-policy RL algorithms:
\begin{itemize}
    \item \textbf{Behavior Guided Policy Gradients (BGPG)}: We propose to replace the KL-based trust region from \cite{trpo} with a WD-based in the behavior space.
    \item \textbf{Behavior Guided Evolution Strategies (BGES)}: BGES improves on Novelty Search \cite{conti} by jointly optimizing for reward and \textit{novelty} using the WD in the behavior space.
\end{itemize}

We also demonstrate a way to harness our methodology for imitation and repulsion learning (Section \ref{section::repulsion_learning}), showing the universality of the proposed techniques.

\section{Motivating Behavior-Guided Reinforcement Learning}

Throughout this paper we prompt the reader to think of a policy as a distribution over its behaviors, induced by the policy's (possibly stochastic) map from state to actions and the unknown environment dynamics. We care about summarizing (or embedding) behaviors into succinct representations that can be compared with each other (via a cost/metric). These comparisons arise naturally when answering questions such as: Has a given trajectory achieved a certain level of reward? Has it visited a certain part of the state space? We think of these summaries or embeddings as characterizing the behavior of the trajectory or relevant state policy-pairs. We formalize these notions in Section \ref{sec:behavior}. 

We show that by identifying policies with the embedding distributions that result of applying the embedding function (summary) to their trajectories, and combining this with the provided cost metric, we can induce a topology over the space of policies given by the Wasserstein distance over their embedding distributions. The methods we propose can be thought of as ways to leverage this ``behavior" geometry for a variety of downstream applications such as policy optimization and imitation learning. 

This topology emerges naturally from the sole definition of an embedding map (behavioral summary) and a cost function. Crucially these choices occur in the semantic space of behaviors as opposed to parameters or visitation frequencies\footnote{If we choose an appropriate embedding map our framework handles visitation frequencies as well. }. One of the advantages of choosing a Wasserstein geometry is that non-surjective trajectory embedding maps are allowed. This is not possible with a KL induced one (in non-surjective cases, computing the likelihood ratios in the KL definition is in general intractable). In Sections \ref{sec:wass} and \ref{sec:wrl} we show that in order to get a handle on this geometry, we can use the dual formulation of the Wasserstein distance to learn functions (Behavioral Test Functions) that can provide scores on trajectories which then can be added to the reward signal (in policy optimization) or used as a reward (in Imitation Learning). 

In summary, by defining an embedding map of trajectories into a behavior embedding space equipped with a metric\footnote{The embedding space can be discrete or continuous and the metric need not be smooth, and can be for example a simple discrete $\{0,1\}$ valued criterion}, our framework allows us to learn ``reward" signals  (Behavioral Test Functions) that can serve to steer policy search algorithms through the ``behavior geometry" either in conjunction with a task specific reward (policy optimization) or on their own (e.g. Imitation Learning). We develop versions of on policy RL algorithms which we call Behavior Guided Policy Gradient (BGPG) and Behavior Guided Evolution Strategies (BGES) that enhance their baseline versions by the use of learned Behavioral Test Functions. Our experiments in Section \ref{sec:exp} show this modification is useful. We also provide a simple example for repulsion learning and Imitation Learning, where we only need access to an expert's embedding. Our framework also has obvious applications to safety, learning policies that avoid undesirable behaviors. 

A final important note is that in this work we only consider simple heuristics for the embeddings, as used in the existing literature. For BGES, these embeddings are those typically used in Quality Diversity algorithms \cite{qualdiv}, while for BGPG we reinterpret the action distribution currently used in KL-based trust regions \cite{schulman2017proximal, trpo}. We emphasize the focus of this paper is on introducing the framework to score these behaviors to guide policy optimization.

\section{Defining Behavior in Reinforcement Learning}
\label{sec:behavior}

A Markov Decision Process ($\mathrm{MDP}$) is a tuple $(\mathcal{S},\mathcal{A},\mathrm{P},\mathrm{R})$. Here $\mathcal{S}$ and $\mathcal{A}$ stand for the sets of states and actions respectively, such that for $s,s^{\prime} \in \mathcal{S}$ and $a \in \mathcal{A}$: $\mathrm{P}(s^{\prime}| a,s)$ is the probability that the system/agent transitions from $s$ to $s^{\prime}$ given action $a$ and $\mathrm{R}(s^{\prime},a,s)$ is a reward obtained by an agent transitioning from $s$ to $s^{\prime}$ via $a$. A policy $\pi_{\theta}:\mathcal{S} \rightarrow \mathcal{A}$ is a (possibly randomized) mapping (parameterized by $\theta \in \mathbb{R}^{d}$) from $\mathcal{S}$ to $\mathcal{A}$. 
Let $\Gamma = \{ \tau = s_0, a_0, r_0, \cdots s_H, a_H, r_H  \text{ s.t. } s_i \in \mathcal{S}, a_i \in \mathcal{A}, r_i \in \mathbb{R}\}$ be the set of possible trajectories enriched by sequences of partial rewards under some policy $\pi$. The undiscounted reward function $\mathcal{R} : \Gamma \rightarrow \mathbb{R}$ (which expectation is to be maximized by optimizing $\theta$) satisfies $\mathcal{R}(\tau) = \sum_{i=0}^H r_i$, where $r_{i}=R(s_{i+1},a_{i},s_{i})$.

\vspace{-2mm}
\subsection{Behavioral Embeddings}\label{subsection::behavioral_embedding}

In this work we identify a policy with what we call a Policy Embedding. We focus on two types of Policy Embeddings both of which are probabilistic in nature, on policy and off policy embeddings, the first being trajectory based and the second ones state-based. 

\subsubsection{On Policy Embeddings}

We start with a Behavioral Embedding Map (BEM), $\Phi: \Gamma \rightarrow \mathcal{E}$, mapping trajectories to embeddings (Fig. ~\ref{figure:bem}), where $\mathcal{E}$ can be seen as a behavioral manifold. On Policy Embeddings can be for example: a) \textbf{State-Based}, such as the final state $\Phi_{1}(\tau) = s_{H}$ b) \textbf{Action-based:} such as the concatenation of actions $\Phi_{4}(\tau)=[a_0,...,a_H]$ or c) \textbf{Reward-based:} the total reward $\Phi_{5}(\tau) = \sum_{t=0}^H r_t$, reward-to-go vector $\Phi_{6}(\tau) = \sum_{t=0}^H r_t \left(\sum_{i= 0}^t e_i   \right)$ (where $e_i \in \mathbb{R}^{H+1}$ is a one-hot vector corresponding to $i$ with dimension index from $0$ to $H$). Importantly, the mapping does not need to be surjective, as we see on the example of the final state embedding.

\label{bems}
\begin{figure}[H]
    \begin{minipage}{0.99\textwidth}
    \centering
    \subfigure{\includegraphics[width=.35\textwidth]{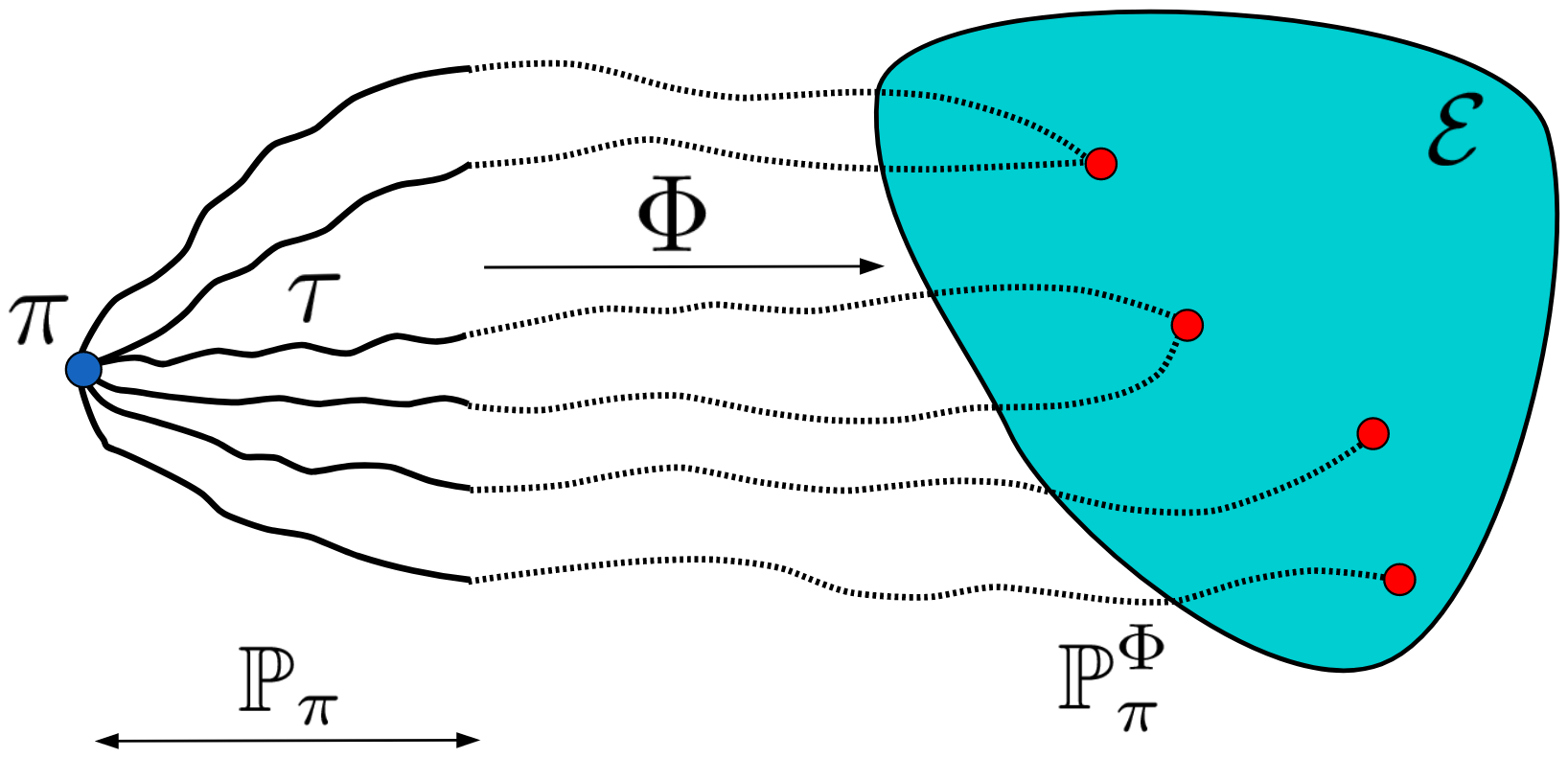}}
    \vspace{-1mm}
    \caption{\small{\textbf{Behavioral Embedding Maps (BEMs)} map trajectories to points in the behavior embedding space $\mathcal{E}$. Two trajectories may map to the same point in $\mathcal{E}$.}}
    \label{figure:bem}
    \end{minipage}
\end{figure}

Given a policy $\pi$, we let $\mathbb{P}_\pi$ denote the distribution induced over the space of trajectories $\Gamma$ and by $\mathbb{P}_\pi^\Phi$ the corresponding pushforward distribution on $\mathcal{E}$ induced by $\Phi$. We call $P_\pi^\Phi$ the policy embeddings of a policy $\pi$. A policy $\pi$ can be fully characterized by the distribution $\mathbb{P}_\pi$ (see: Fig. \ref{figure:bem}).

Additionally, we require $\mathcal{E}$ to be equipped with a metric (or cost function) $C:\mathcal{E} \times \mathcal{E} \rightarrow \mathbb{R}$. Given two trajectories $\tau_1, \tau_2$ in $\Gamma$, $C(\Phi(\tau_1), \Phi(\tau_2))$ measures how different these trajectories are in the behavior space. We note that some embeddings are only for the tabular case ($|\mathcal{S}|,|\mathcal{A}| < \infty$) while others are universal. 

\subsubsection{Off Policy Embeddings}

Let $\mathbb{P}_\mathcal{S}$ be some ``probe" distribution over states $\mathcal{S}$ and $\pi$ be a policy. We define $\mathbb{P}_\pi^{\Phi_\mathcal{S}}$ to be the distribution of pairs $(s, \pi(s))$ for $s \sim \mathbb{P}_S$. We identify $\mathcal{E}$ with the product space $\mathcal{S}\times \Delta_\mathcal{A}$ (where $\Delta_{\mathcal{A}}$ denotes the set of distributions over $\mathcal{A}$) endowed with an appropriate metric $C: \mathcal{E} \times \mathcal{E} \rightarrow \mathbb{R}$.In our experiments we identify $\mathcal{C}$ with the $l_2$ norm over $\mathcal{E}$ and $\mathbb{P}_\mathcal{S}$ with a mechanism that samples states from a buffer of visited states. We only add an $\mathcal{S}$ to the notation for $\mathbb{P}_\pi^{\Phi_\mathcal{S}}$ when distinguishing from on-policy embeddings is needed. 

This definition allows the ``probing" distribution $\mathbb{P}_\mathcal{S}$ to be off policy, independent of the policy at hand.  If $C$ is a norm and $\mathbb{P}_{\mathcal{S}}$ has mass only in user-relevant areas of the state space, a WD of zero between two policies (whose embeddings use the same probing distribution) implies they behave equally where the user cares. Our off Policy Embeddings are of the form $(s, \pi(s))$ but other choices are valid.

\vspace{-1mm}

\section{Wasserstein Distance \& Optimal Transport Problem}
\label{sec:wass}

Let $\mu, \nu$ be (Radon) probability measures over domains $\mathcal{X} \subseteq \mathbb{R}^m, \mathcal{Y} \subseteq \mathbb{R}^n$ and let $\mathcal{C} : \mathcal{X} \times \mathcal{Y} \rightarrow \mathbb{R}$ be a cost function. For $\gamma > 0 $, a \textit{smoothed} Wasserstein Distance is defined as:
\vspace{-3mm}
\begin{equation}\label{equation::definition_cont_wass}
    \mathrm{WD}_\gamma(\mu, \nu ) :=  \min_{\pi \in \Pi(\mu, \nu)} \int_{\mathcal{X} \times \mathcal{Y}}C(\mathbf{x},\mathbf{y}) d\pi(\mathbf{x},\mathbf{y}) + \Sigma,
\end{equation} 
\vspace{-3mm}

\vspace{-2mm}
where $\Sigma=\gamma \text{KL}(\pi | \xi )$,
$\Pi(\mu, \nu)$ is the space of couplings (joint distributions) over $\mathcal{X} \times \mathcal{Y}$ with marginal distributions $\mu$ and $\nu$, $\text{KL}(\cdot | \cdot)$ denotes the KL divergence between distributions $\pi$ and $\rho$ with support $\mathcal{X}\times \mathcal{Y}$ defined as:
$
    \mathrm{KL}(\pi | \rho ) = \int_{\mathcal{X} \times \mathcal{Y}} \left( \log\left( \frac{ d\pi}{d\xi  }(\mathbf{x},\mathbf{y})    \right)     \right) d\pi(\mathbf{x},\mathbf{y})
$ and $\xi$ is a reference measure over $\mathcal{X} \times \mathcal{Y}$. When the cost is an $\ell_p$ distance  and $\gamma= 0$, $\mathrm{WD}_\gamma$ is also known as the Earth mover's distance and the corresponding optimization problem is known as the \emph{optimal transport problem} (OTP).

\subsection{Wasserstein Distance: Dual Formulation}

We will use smoothed WDs to derive efficient regularizers for RL algorithms.  To arrive at this goal, we first need to consider the dual form of Equation \ref{equation::definition_cont_wass}. 
Under the subspace topology \citep{bourbaki1966general} for $\mathcal{X}$ and $\mathcal{Y}$, let $\mathcal{C}(\mathcal{X})$ and $\mathcal{C}(\mathcal{Y})$ denote the space of continuous functions over $\mathcal{X}$ and $\mathcal{Y}$ respectively. The choice of the subspace topology ensures our discussion encompasses the discrete case.

 Let $C: \mathcal{X} \times \mathcal{Y} \rightarrow \mathbb{R}$ be a cost function, interpreted as the ``ground cost'' to move a unit of mass from $x$ to $y$. Define $\mathbb{I}$ as the function outputting values of its input predicates.
 Using Fenchel duality, we can obtain the following dual formulation of the problem in Eq.~\ref{equation::definition_cont_wass}:
\begin{equation}
\label{equation::dual_formulation_wass}
\mathrm{WD}_\gamma(\mu, \nu) = \max_{\lambda_\mu \in \mathcal{C}(\mathcal{X}), \lambda_\nu \in \mathcal{C}(\mathcal{Y}) } \Psi(\lambda_\mu, \lambda_\nu),    
\end{equation}
where $\Psi(\lambda_\mu, \lambda_\nu)=\int_{\mathcal{X}} \lambda_\mu( \mathbf{x}) d\mu(\mathbf{x}) \notag 
    - \int_{\mathcal{Y}} \lambda_\nu(\mathbf{y}) d\nu(\mathbf{y}) - E_C(\lambda_\mu, \lambda_\nu)$ 
and the damping term $E_{C}(\lambda_\mu, \lambda_\nu)$ equals:
\begin{small}
\begin{equation}
\label{equation::dual_penalty_continuous}
E_C(\lambda_\mu, \lambda_\nu) = 
\mathbb{I}(\gamma > 0) \int_{\mathcal{X} \times \mathcal{Y} } \rho(\mathbf{x},\mathbf{y}) d\xi(\mathbf{x}, \mathbf{y})+\mathbb{I}(\gamma=0)\mathbb{I}(\mathcal{A})
\end{equation}
\end{small}
\vspace{-3mm}

for $\rho(\mathbf{x},\mathbf{y})=\gamma \exp (\frac{ \lambda_\mu(\mathbf{x}) - \lambda_\nu(\mathbf{y}) - C(\mathbf{x},\mathbf{y})   }{   \gamma })$
and $\mathcal{A}=[(\lambda_\mu, \lambda_\nu) \in \{ (u,v) \text{ s.t. } \forall (\mathbf{x},\mathbf{y}) \in \mathcal{X} \times \mathcal{Y}: 
                                     u(\mathbf{x}) - v(\mathbf{y}) \leq C(\mathbf{x},\mathbf{y}) \}$].
                                    
 We will set the damping distribution $d\xi(\mathbf{x},\mathbf{y}) \propto 1$ for discrete domains and $d\xi(\mathbf{x},\mathbf{y}) = d\mu(\mathbf{x})d\nu(\mathbf{y})$ otherwise. 

If $\lambda_\mu^* , \lambda_\nu^*$ are the functions achieving the maximum in Eq.~\ref{equation::dual_formulation_wass}, and $\gamma$ is sufficiently small then $\mathrm{WD}_\gamma(\mu, \nu) \approx \mathbb{E}_\mu\left[    \lambda_\mu^*(\mathbf{x} ) \right]  - \mathbb{E}_\nu\left[  \lambda^*_\nu( \mathbf{y})  \right]$, with equality when $\gamma = 0$. When for example $\gamma = 0$, $\mathcal{X} = \mathcal{Y}$, and $C(x,x)= 0$ for all $x \in \mathcal{X}$, it is easy to see $\lambda^*_\mu(x) = \lambda^*_\nu(x) = \lambda^*(x)$ for all $x \in \mathcal{X}$. In this case the difference between $ \mathbb{E}_\mu\left[    \lambda^*(\mathbf{x} ) \right]$  and $\mathbb{E}_\mu\left[  \lambda^*( \mathbf{y})  \right]$ equals the WD. In other words, the function $\lambda^*$ gives higher scores to regions of the space $\mathcal{X}$ where $\mu$ has more mass. This observation is key to the success of our algorithms in guiding optimization towards desired behaviors.
\vspace{-.3cm}
\subsection{Computing $\lambda_\mu^*$ and $\lambda_\nu^*$}

We combine several techniques to make the optimization of objective from Eq.~\ref{equation::dual_formulation_wass} tractable. First, we replace $\mathcal{X}$ and $\mathcal{Y}$ with the functions from a RKHS corresponding to universal kernels \citep{micchelli}. This is justified since those function classes are dense in the set of continuous functions of their ambient spaces. In this paper we choose the RBF kernel and approximate it using random Fourier feature maps \citep{randomfeatures2007} to increase efficiency. Consequently, the functions $\lambda$ learned by our algorithms have the following form: $\lambda(\mathbf{x})= (\mathbf{p}^\lambda)^{\top}\phi(\mathbf{x})$, where $\phi$ is a random feature map with $m$ standing for the number of random features and $\mathbf{p}^\lambda \in \mathbb{R}^{m}$. For the RBF kernel, $\phi$ is defined as follows: $\phi(\mathbf{z}) = \frac{1}{\sqrt{m}} \cos(\mathbf{Gz}+\mathbf{b})$ for $\mathbf{z} \in \mathbb{R}^{d}$, where $\mathbf{G} \in \mathbb{R}^{m \times d}$ is Gaussian with iid entries taken from $\mathcal{N}(0,1)$, $b \in \mathbb{R}^{m}$ with iid $b_{i}$s such that $b_{i} \sim \mathrm{Unif}[0, 2\pi]$ and the $\cos$ function acts elementwise.

\begin{figure}[H]
    \begin{minipage}{0.99\textwidth}
    \centering
    \subfigure{\includegraphics[width=.35\textwidth]{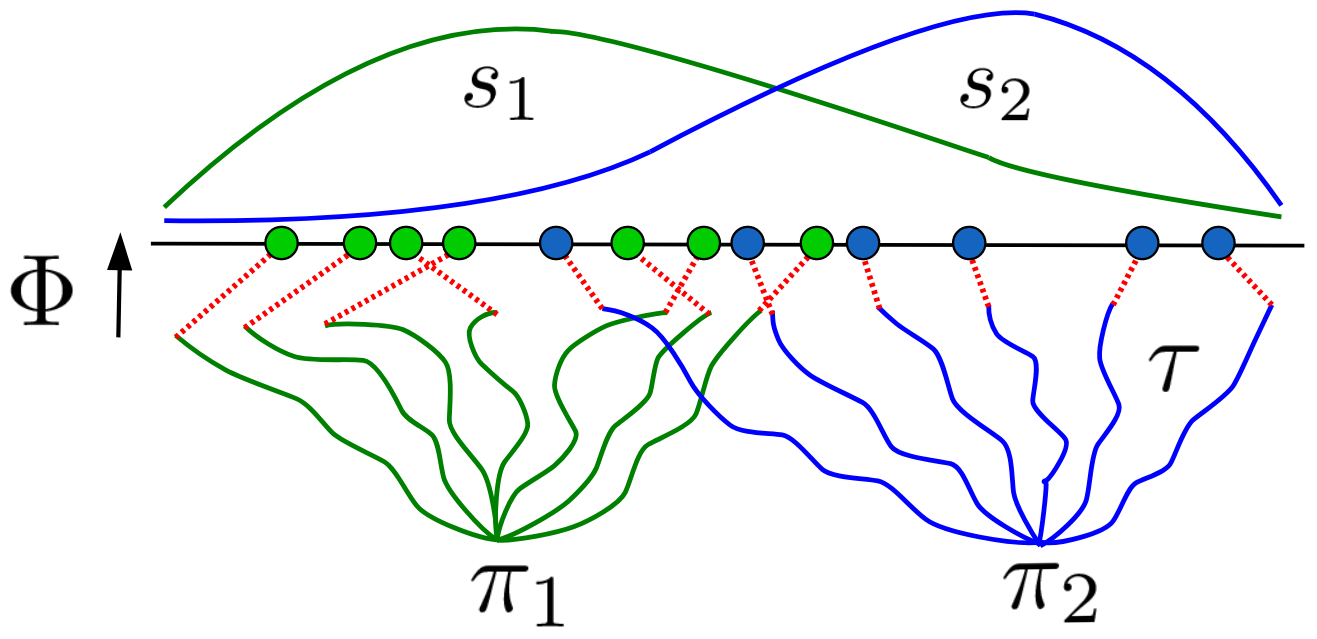}}
    \vspace{-1mm}
    \caption{\small{Two policies $\pi_1$ (green) and $\pi_2$ (blue) whose BEMs map trajectories to points in the real line.}}
    \label{figure:bef}
    \end{minipage}
\end{figure}

Henceforth, when we refer to optimization over $\lambda$, we mean optimizing over corresponding dual vectors $\mathbf{p}^{\lambda}$ associated with $\lambda$. We can solve for the optimal dual functions by running Stochastic Gradient Descent (SGD) over the dual objective in Eq.~\ref{equation::dual_formulation_wass}. Algorithm \ref{Alg:wass_constinuous_SGD} is the random features

\begin{algorithm}[tb]
\textbf{Input:} kernels $\kappa, \mathcal{\ell}$ over $\mathcal{X}, \mathcal{Y}$ respectively with corresponding random feature maps $\phi_\kappa, \phi_{\mathcal{\ell}}$, smoothing parameter $\gamma$, gradient step size $\alpha$, number of optimization rounds $M$, initial dual vectors $\mathbf{p}_{0}^\mu, \mathbf{p}_{0}^\nu$. \\
\For{$t =0, \cdots , M$ }{
   1. Sample $(x_t, y_t) \sim \mu \bigotimes \nu$. \\
   2. Update:      $ \binom{\mathbf{p}^\mu_t}{\mathbf{p}_t^\nu}$ using Equation \ref{equation::SGD_random_features}.\\
}
\textbf{Return:} $\mathbf{p}_{M}^\mu, \mathbf{p}_{ M}^\nu$. 
 \caption{Random Features Wasserstein SGD}
\label{Alg:wass_constinuous_SGD}
\end{algorithm}
\vspace{-3mm}
equivalent of Algorithm 3 in \cite{genevay2016stochastic}.
Given input kernels $\kappa, \ell$ and a fresh sample $(x_t, y_t) \sim \mu \bigotimes \nu$ the SGD step w.r.t. the current iterates $\mathbf{p}_{t-1}^\mu, \mathbf{p}_{t-1}^\nu$ satisfies:
\begin{small}
\begin{align}
    F(\mathbf{p}_1, \mathbf{p}_2, x,y ) &= \exp\left( \frac{ (\mathbf{p}_1)^\top \phi_\kappa(x)  -  (\mathbf{p}_2)^\top \phi_{\mathcal{\ell}}(x)  - C(x,y) }{\gamma} \right)\notag\\
  \binom{\mathbf{p}^\mu_{t+1}}{\mathbf{p}_{t+1}^\nu} &= \binom{  \mathbf{p}^\mu_{t} }{\mathbf{p}^\nu_{t}}
      +  \left(1- F( \mathbf{p}_{t}^\mu, \mathbf{p}_{t}^\nu, x_t, y_t) \right)v_t,\label{equation::SGD_random_features}
\end{align}
\end{small}
where $v_t =\frac{\alpha}{\sqrt{t}}(\phi_\kappa(x_t), -\phi_\ell(y_t)  )^\top $. An explanation and proof of these formulae is in Lemma \ref{lemma::gradient_random_features} in the Appendix. If $\mathbf{p}_*^\mu, \mathbf{p}_*^{\nu}$ are the optimal dual vectors, $p_* = (\mathbf{p}_*^\mu, \mathbf{p}_*^\nu)^\top$, $(x_1, y_1), \cdots, (x_k, y_k) \stackrel{\text{i.i.d}}{\sim} \mu \bigotimes \nu$, $\mathbf{v}_i^{\kappa, \ell} = (\phi_\kappa(x_i), -\phi_\ell(y_i))^\top$ for all $i$, and $\hat{\mathbb{E}}$ denotes the empirical expectation over the $k$ samples $\{(x_i, y_i)\}_{i=1}^k$, Algorithm 1 can be used to get an estimator of $\mathrm{WD}_\gamma(\mu, \nu)$ as:
\begin{small}
\begin{align}
\label{equation::empirical_smooth_wass} 
 \widehat{\mathrm{WD}}_\gamma( \mu, \nu ) &=  \hat{\mathbb{E}}\left[  \langle \mathbf{p}_*, \mathbf{v}_i^{\kappa, \ell} \rangle -  \frac{F(\mathbf{p}_*^\mu, \mathbf{p}_*^\nu, x_i, y_i)}{\gamma}\right]
\end{align}
\end{small}
\vspace{-5mm}

\vspace{-3mm}
\section{Behavior-Guided Reinforcement Learning}
\label{sec:wrl}
We explain now how to get practical algorithms based on the presented methods. Denote by $\pi_\theta$ a policy parameterized by $\theta \in \mathbb{R}^d$. The goal of policy optimization algorithms is to find a policy maximizing, as a function of the policy parameters, the expected total reward $ \mathcal{L}(\theta) := \mathbb{E}_{\tau \sim\mathbb{P}_{\pi_\theta}}\left[\mathcal{R}(\tau)   \right]$.

\vspace{-3mm}
\subsection{Behavioral Test Functions}\label{subsection::behavioral_test_functions}

If $C : \mathcal{E} \times \mathcal{E} \rightarrow \mathbb{R}$ is a cost function defined over behavior space $\mathcal{E}$, and $\pi_1, \pi_2$ are two policies, then in the case of \textbf{On-Policy} Embeddings:
\begin{small}
\begin{equation*}\label{equation::behavioral_test_functions_intuition}
    \mathrm{WD}_\gamma(\mathbb{P}_{\pi_1}^\Phi, \mathbb{P}_{\pi_2}^\Phi) \approx \mathop{\mathbb{E}}_{\tau \sim \mathbb{P}_{\pi_1}}\left[ \lambda_1^*(\Phi(\tau))    \right] - \mathop{\mathbb{E}}_{\tau \sim \mathbb{P}_{\pi_2}} \left[   \lambda_2^*(\Phi(\tau))    \right],
\end{equation*}
\end{small}
where $\lambda_1^*, \lambda_2^*$ are the optimal dual functions. The maps $s_1 := \lambda_1^* \circ \Phi : \Gamma \rightarrow \mathbb{R}$ and $s_2 := \lambda_2^* \circ \Phi : \Gamma \rightarrow \mathbb{R}$ define score functions over the space of trajectories. If $\gamma$ is close to zero, the score function $s_i$ gives higher scores to trajectories from $\pi_i$ whose behavioral embedding is common under $\pi_i$ but rarely appears under $\pi_j$ for $j \neq i$ (Fig. \ref{figure:bef}).
In the case of \textbf{Off-Policy} Embeddings:
\begin{small}
\begin{equation*}\label{equation::behavioral_test_functions_intuition}
    \mathrm{WD}_\gamma(\mathbb{P}_{\pi_1}^{\Phi_{\mathcal{S}}}, \mathbb{P}_{\pi_2}^{\Phi'_{\mathcal{S}}}) \approx \mathop{\mathbb{E}}_{S \sim \mathbb{P}_\mathcal{S}}\left[ \lambda_1^*(S, \pi_1(S))    \right] - \mathop{\mathbb{E}}_{S \sim \mathbb{P}'_\mathcal{S}} \left[   \lambda_2^*(S, \pi_2(S))    \right],
\end{equation*}
\end{small}
where $\lambda_1^*, \lambda_2^*$ are maps from state policy pairs $(S, \pi_1(S))$ to scores, and $\mathbb{P}_{\mathcal{S}}, \mathbb{P}_{\mathcal{S}}'$ are probing distributions.

\subsection{Repulsion and Imitation Learning}
\label{section::repulsion_learning}

To illustrate the intuition behind behavioral test functions and on policy embeddings, we introduce an algorithm for multi-policy repulsion learning based on our framework. 
Algorithm \ref{Alg:repulsion_learning} maintains two policies $\pi^\mathbf{a}$ and $\pi^{\mathbf{b}}$.
\vspace{-3mm}

\begin{algorithm}[H]
\textbf{Input: } $\beta, \eta > 0$, $M\in \mathbb{N}$ \\
\textbf{Initialize:} Initial stochastic policies $\pi_0^\mathbf{a}, \pi_0^\mathbf{b}$, parametrized by $\theta_0^\mathbf{a}, \theta_0^\mathbf{b} $ respectively, Behavioral Test Functions $\lambda_1^\mathbf{a}, \lambda_2^\mathbf{b}$\; \\
\For{$t = 1, \dots , T$ }{
   1. Collect $\{ \tau_i^\mathbf{a} \}_{i=1}^M\sim\mathbb{P}_{\pi_{t-1}^\mathbf{a}}$ and $\{ \tau_i^\mathbf{b}  \}_{i=1}^M \sim \mathbb{P}_{\pi_{t-1}^\mathbf{b}}$.\\
   2. Form $\tilde{R}_\mathbf{c}(\tau_1, \tau_2)$ for $\mathbf{c} \in \{ \mathbf{a}, \mathbf{b}\}$ using Equation \ref{equation::repulsion_reward_augmentation}.\; \\
  3. For $\mathbf{c} \in \{\mathbf{a}, \mathbf{b} \}$ and $(\tau_1, \tau_2) \sim \{ \tau_i^\mathbf{a} \}_{i=1}^M \times \{ \tau_i^\mathbf{b} \}_{i=1}^M$ use REINFORCE~\cite{williams1992simple} to perform update:
  \begin{align*}
       \theta_t^{\mathbf{c}} = \theta_{t-1}^{\mathbf{c}} + \eta \nabla_{\theta} \tilde{R}_\mathbf{c}(\tau_1, \tau_2)
  \end{align*}
  5. Update $\lambda_1^\mathbf{a}, \lambda_2^\mathbf{b}$ with $\{ \tau_i^\mathbf{a}, \tau_i^\mathbf{b} \}_{i=1}^M$ via Algorithm \ref{Alg:wass_constinuous_SGD}.
}
 \caption{Behvaior-Guided Repulsion Learning}
\label{Alg:repulsion_learning}
\end{algorithm}
\vspace{-6mm}

Each policy is optimized by taking a policy gradient step (using the REINFORCE gradient estimator \cite{williams1992simple}) to optimize surrogate rewards $\tilde{\mathcal{R}}_\mathbf{a}$ and $\tilde{\mathcal{R}}_\mathbf{b}$. 

\vspace{-5mm}
\begin{figure}[H]
    \centering\begin{minipage}{0.7\textwidth}
	\centering\subfigure[$\pi_0^\mathbf{a}$]{\includegraphics[width=0.25\textwidth]{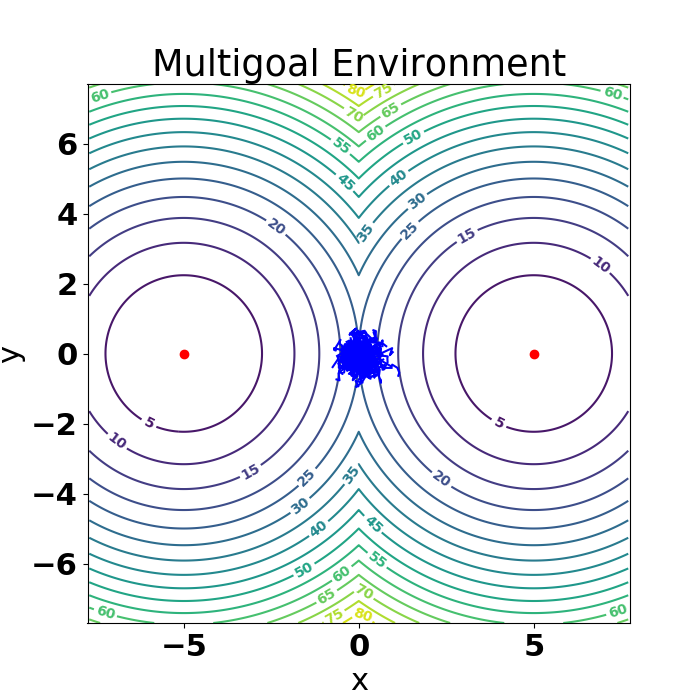}}	
	\centering\subfigure[$\pi_0^\mathbf{b}$]{\includegraphics[width=0.25\textwidth]{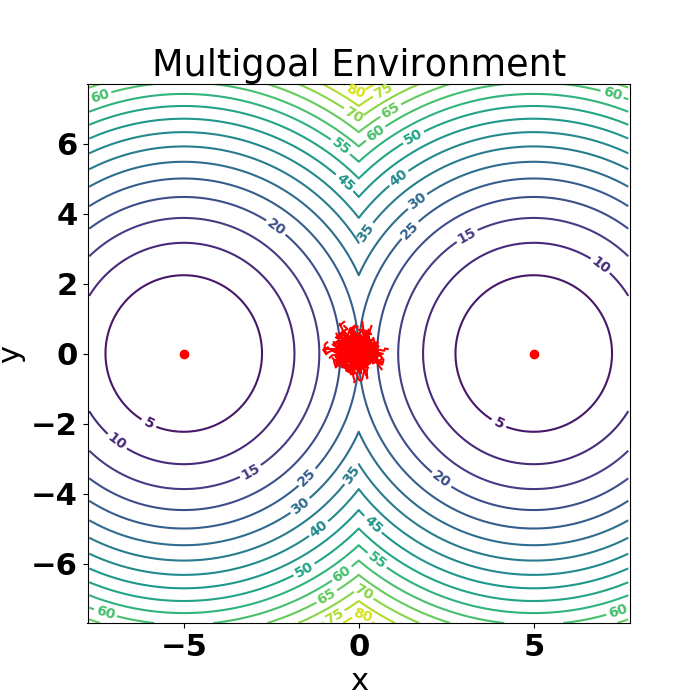}}
	\centering\subfigure[$\lambda^{\mathbf{a}}$ \& $-\lambda^{\mathbf{b}}$, $t=0$]{\includegraphics[width=0.36\textwidth]{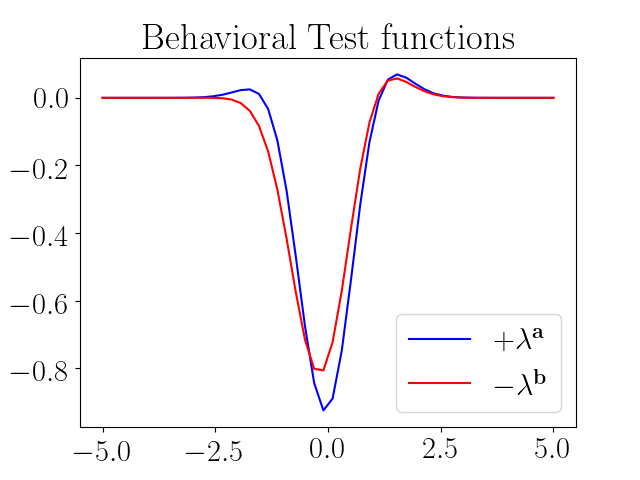}}
	\end{minipage}
    \begin{minipage}{0.7\textwidth}
	\centering\subfigure[$\pi_{22}^\mathbf{a}$]{\includegraphics[width=0.25\textwidth]{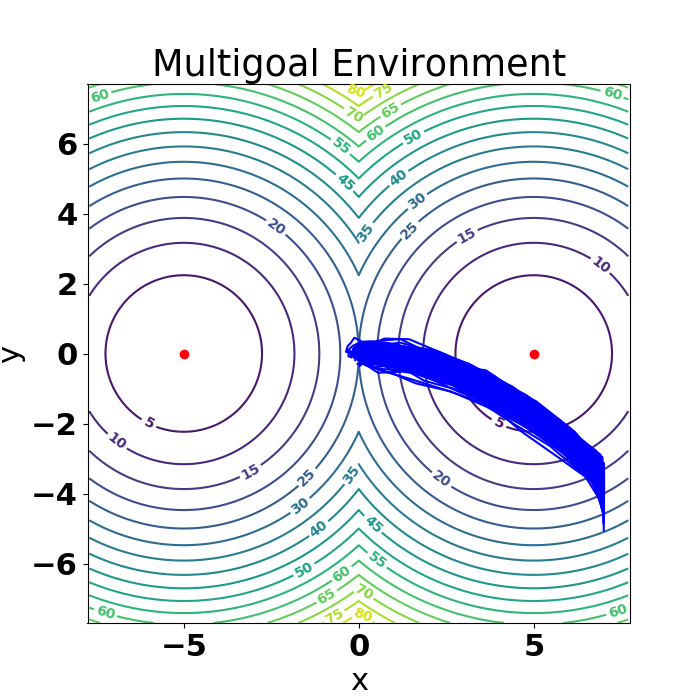}}	
	\centering\subfigure[$\pi_{22}^\mathbf{b}$]{\includegraphics[width=0.25\textwidth]{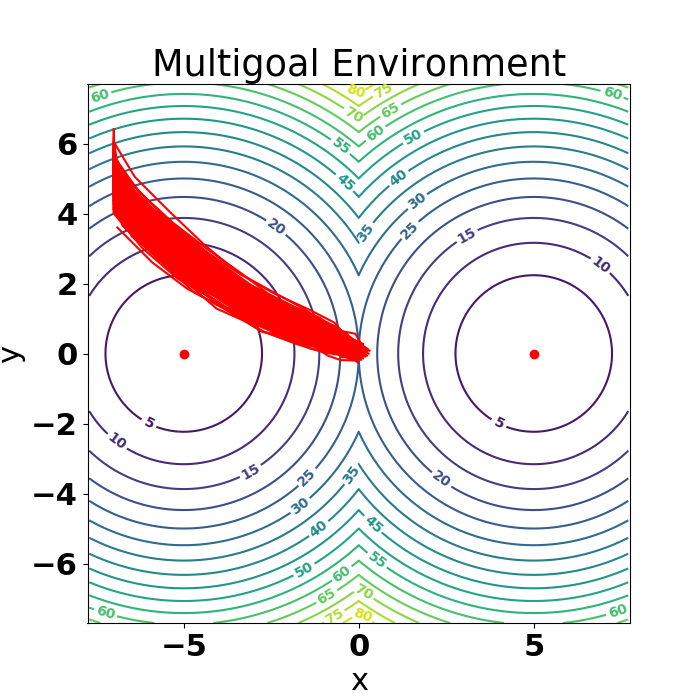}}
	\centering\subfigure[ $\lambda^{\mathbf{a}}$ \& $-\lambda^{\mathbf{b}}$, $t=22$]{\includegraphics[width=0.36\textwidth]{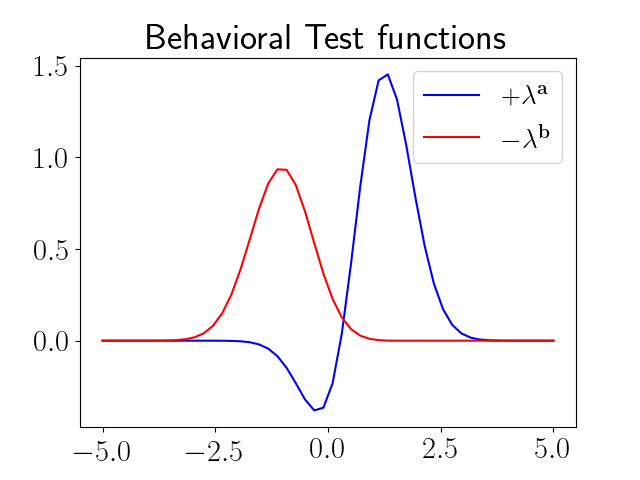}}
	\end{minipage}
	\begin{minipage}{0.7\textwidth}
	\centering\subfigure[$\pi_{118}^\mathbf{a}$]{\includegraphics[width=0.25\textwidth]{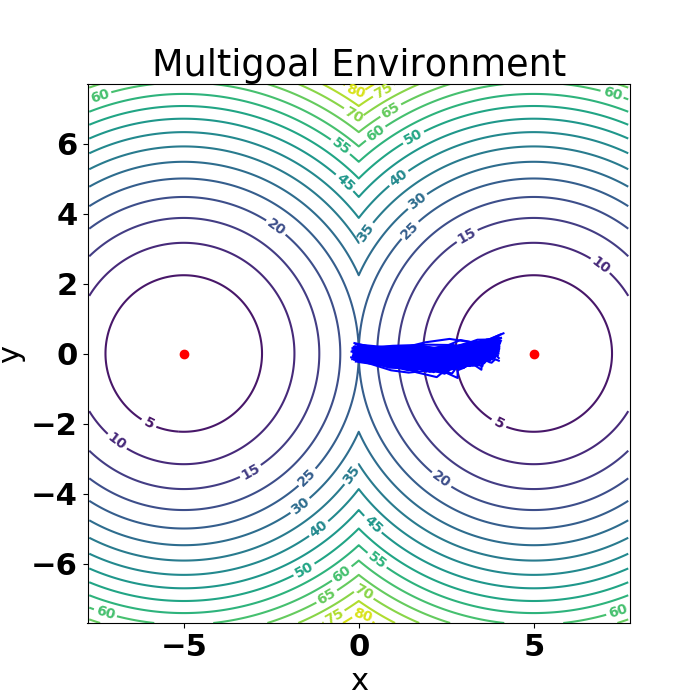}}	
	\centering\subfigure[$\pi_{118}^\mathbf{b}$]{\includegraphics[width=0.25\textwidth]{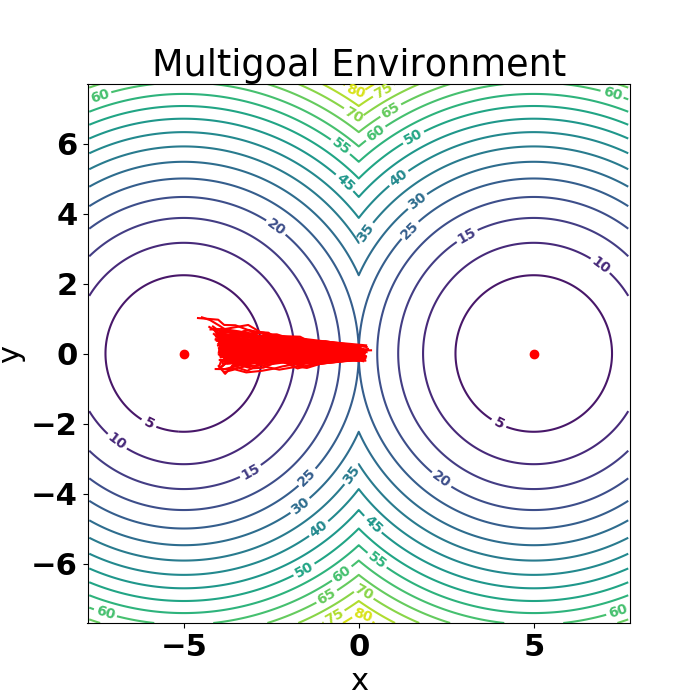}}
	\centering\subfigure[$\lambda^{\mathbf{a}}$ \& $-\lambda^{\mathbf{b}}$, $t=118$]{\includegraphics[width=0.36\textwidth]{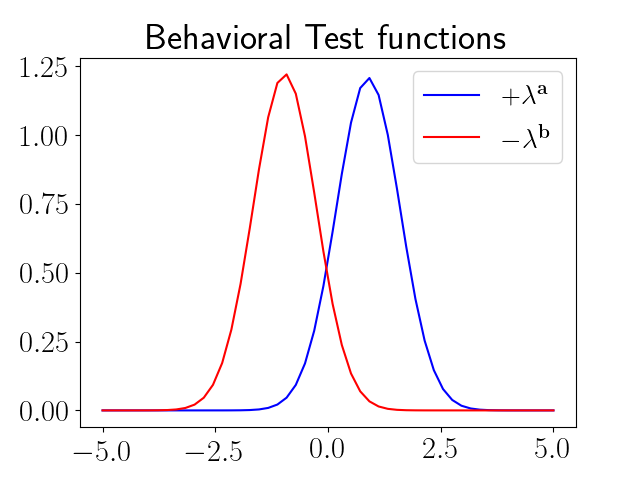}}
	\end{minipage}
	\vspace{-1mm}
	\caption{\small{a) and b) Initial state of policies $\pi^{\mathbf{a}}, \pi^{\mathbf{b}}$ and Test functions $\lambda^{\mathbf{a}}, \lambda^{\mathbf{b}}$. d)-i) Policy evolution and Test Functions.}}
	\label{fig:repulsionplot} 
\end{figure}
\vspace{-5mm}

These combine the signal from the task's reward function $\mathcal{R}$ and the repulsion score encoded by the input BEM $\Phi$ and behavioral test functions $\lambda^{\mathbf{a}}$ and $\lambda^{\mathbf{b}}$:
\begin{equation}\label{equation::repulsion_reward_augmentation}
\tilde{R}_\mathbf{c}(\tau_\mathbf{a}, \tau_\mathbf{b}) = \mathcal{R}(\tau_\mathbf{c}) + \beta \widehat{\mathrm{WD}}_\gamma( \mathbb{P}_{\pi^\mathbf{a}}^\Phi, \mathbb{P}_{\pi^{\mathbf{a}}}^\Phi ), \mathbf{c}\in \{ \mathbf{a}, \mathbf{b}\}
\end{equation}
\vspace{-6mm}

We test Algorithm \ref{Alg:repulsion_learning} on an environment consisting of a particle that needs to reach one of two goals on the plane. Policies outputs a velocity vector and stochasticity is achieved by adding Gaussian noise to it. The embedding $\Phi$ maps trajectories $\tau$ to their mean displacement along the $x-$axis. Fig. \ref{fig:repulsionplot}  shows how the policies' behavior evolves throughout optimization and how the Test Functions guide the optimization by favouring the two policies to be far apart.  The experiment details are in the Appendix (Section \ref{Sec:repulsion}). A related guided trajectory scoring approach to imitation learning is explored in Appendix \ref{section::appendix_imitation}. 

\subsection{Algorithms}
\label{sec:algorithms}

We propose to solve a WD-regularized objective to tackle behavior-guided policy optimization. All of our algorithms hinge on trying to maximize an objective of the form:

\vspace{-4mm}
\begin{equation}
    F(\theta) = \mathcal{L}(\theta) + \beta \mathrm{WD}_\gamma(\mathbb{P}_{\pi_\theta}^\Phi ,\mathbb{P}^{\Phi}_{\mathrm{b}}), \label{eq:rlobj}
\vspace{-2mm}    
\end{equation}

where $\mathbb{P}_\mathrm{b}^\Phi$ is a base distribution\footnote{Possibly using off policy embeddings.} over behavioral embeddings (possibly dependent on $\theta$) and $\beta \in \mathbb{R}$ could be positive or negative. Although the base distribution $\mathbb{P}_\mathrm{b}^\Phi$ could be arbitrary, our algorithms will instantiate $\mathbb{P}_\mathrm{b}^\Phi = \frac{1}{| \mathcal{S}|} \mathop{\cup}_{\pi' \in \mathcal{S}} \mathbb{P}_{\pi'}^\Phi$ for some family of policies $\mathcal{S}$ (possibly satisfying $| \mathcal{S} | = 1$) we want the optimization to attract to / repel from. 

In order to compute approximate gradients for $F$, we rely on the dual formulation of the WD. After substituting the composition maps resulting from Eq.~\ref{equation::behavioral_test_functions_intuition} into Eq.~\ref{eq:rlobj}, we obtain, for \textbf{on-policy} embeddings:
\begin{equation}\label{equation::approximate_combined_objective}
     F(\theta) \approx \mathbb{E}_{\tau \sim \mathbb{P}_{\pi_\theta}}\left[ \mathcal{R}(\tau) +  \beta s_1(\tau)      \right] - \beta \mathbb{E}_{\phi \sim \mathbb{P}_{\mathrm{b}}^\Phi}\left[ 
    \lambda_2^*(\phi)   \right], 
\end{equation}
where $s_1: \Gamma \rightarrow \mathbb{R}$ equals $s_1 = \lambda_1^* \circ \Phi$, the Behavioral Test Function of policy $\pi_\theta$ and $\lambda_2^*$ is the optimal dual function of embedding distribution
$\mathbb{P}_{\mathbf{b}}^\Phi$. 
Consequently $ \nabla_\theta F(\theta) \approx \nabla_\theta  \mathbb{E}_{\tau \sim \mathbb{P}_{\pi_\theta}}\left[ \mathcal{R}(\tau) + \beta  s_1(\tau)      \right]$. We learn a score function $s_1$ over trajectories that can guide our optimization by favoring those trajectories that show desired global behaviors. For \textbf{off-policy} embeddings, with state probing distributions $\mathbb{P}_\mathcal{S}$ and $\mathbb{P}_{\mathcal{S}}^b$  the analogous to Equation \ref{equation::approximate_combined_objective_off_policy} is:
\begin{align}
     F(\theta) &\approx \mathbb{E}_{\tau \sim \mathbb{P}_{\pi_\theta}}\left[ \mathcal{R}(\tau) \right]  +  \beta\mathbb{E}_{(S, \pi_\theta(S)) \sim \mathbb{P}^{\Phi_\mathcal{S}} } \left[ \lambda_1^*(S, \pi_\theta(S)) \right] \notag\\
     &\quad -\beta\mathbb{E}_{(S, \pi_b(S)) \sim \mathbb{P}_{\mathrm{b}}^{\Phi_{\mathcal{S}}}}\left[ 
    \lambda_2^*(S, \pi_b(S))   \right], \label{equation::approximate_combined_objective_off_policy}
\end{align}
Consequently, if $\mathbb{P}_\mathrm{b}^{\Phi_\mathcal{S}}$ is independent from $\theta$:
\begin{equation*}
\nabla_\theta F(\theta) \approx \nabla_\theta \mathbb{E}_{\tau \sim \mathbb{P}_{\pi_\theta}}\left[ \mathcal{R}(\tau) \right] + \beta \mathbb{E}_{s \sim \mathbb{P}_{\mathcal{S}}}\left[ \nabla_\theta \lambda_1^*(s, \pi_\theta(s))  \right] .
\end{equation*}
Eq.~\ref{equation::approximate_combined_objective} and \ref{equation::approximate_combined_objective_off_policy} are approximations to the true objective from Eq.~\ref{eq:rlobj} whenever $\gamma > 0$. In practice, the entropy regularization requires a damping term $E_C(\lambda_1^*, \lambda_2^*)$ as defined in Equation \ref{equation::dual_penalty_continuous}. If $\xi( \mathbb{P}_{\pi_\theta}^\Phi, \mathbb{P}_{\mathrm{b}}^\Phi )$ is the damping joint distribution of choice  
 and $\rho(\phi_1, \phi_2) =\gamma \exp\left( \frac{ \lambda_{\pi_\theta}( \phi_1) - \lambda_\mathrm{b}( \phi_2 ) - C(\phi_1, \phi_2)   }{\gamma }          \right)  $ (for off policy embeddings $\phi$ is a state policy pair $(S, \pi(S))$), the damping term equals:  $\mathbb{E}_{\phi_1, \phi_2 \sim \xi( \mathbb{P}_{\pi_\theta}^\Phi, \mathbb{P}_{\mathrm{b}}^\Phi ) }\left[     \rho(\phi_1, \phi_2) \right]$. Gradients $\nabla_\theta$ through $E_C$ can be derived using a similar logic as the gradients above. When the embedding space $\mathcal{E}$ is not discrete and $\mathbb{P}_\mathrm{b}^\Phi = \mathbb{P}_\pi^\Phi$ for some policy $\pi$, we let $\xi(\mathbb{P}_{\pi_\theta}^\Phi, \mathbb{P}_{\mathrm{b}}^\Phi    ) = \mathbb{P}_{\pi_\theta}^\Phi \bigotimes \mathbb{P}_\pi^\Phi$, otherwise  $\xi(\mathbb{P}_{\pi_\theta}^\Phi, \mathbb{P}_{\mathrm{b}}^\Phi    )  = \frac{1}{|\mathcal{E}|^2}\mathrm{1}$, a uniform distribution over $\mathcal{E} \times \mathcal{E}$.

All of our methods perform a version of alternating SGD optimization: we take certain number of SGD steps over the internal dual Wasserstein objective, followed by more SGD steps over the outer objective having fixed the test functions.

We consider two approaches to optimizing this objective. Behavior-Guided Policy Gradient (BGPG) explores in the action space as in policy gradient methods \citep{trpo, schulman2017proximal}, while Behavior-Guided Evolution Strategies (BGES) considers a black-box optimization problem as in Evolution Strategies (ES, \cite{ES}).

\subsection{Behavior-Guided Policy Gradient (BGPG)}
\label{sec:bgpg}

Here we present the Behavior-Guided Policy Gradient (BGPG) algorithm (Alg. \ref{Alg:bgpg}). Specifically, we maintain a stochastic policy $\pi_\theta$ and compute policy gradients as in prior work \citep{trpo}. 

\vspace{-2mm}
\begin{algorithm}[h]
\textbf{Input: } Initialize stochastic policy $\pi_0$ parametrized by $\theta_0$, $\beta<0, \eta > 0$, $M \in \mathbb{N}$ \\
\For{$t = 1, \dots , T$ }{
   1. Run $\pi_{t-1}$ in the environment to get advantage values $A^{\pi_{t-1}}(s,a)$ and trajectories $\{\tau_i^{(t)} \}_{i=1}^{M}$ \\
   2. Update policy and test functions via several alternating policy gradient steps over $F(\theta)$. \;\\
 3. Use samples from  $\mathbb{P}_{\pi_{{t-1}}} \bigotimes \mathbb{P}_{\pi_\theta} $ and Algorithm \ref{Alg:wass_constinuous_SGD}  to update $\lambda_1, \lambda_2$ and  take SGA step $\theta_{t} = \theta_{t-1} + \eta \hat{\nabla}_\theta \hat{F}(\theta_{t-1})$
}
 \caption{Behavior-Guided Policy Gradient}
\label{Alg:bgpg}
\end{algorithm}

For \textbf{on-policy} embeddings the objective function $F(\theta)$ takes the form:
\vspace{-3mm}
\begin{small}
\begin{equation}
    F(\theta) = \mathop{\mathbb{E}}_{\tau_1, \tau_2  \sim  \mathbb{P}_{\pi_{{t-1}}} \bigotimes \mathbb{P}_{\pi_\theta}     }\Big[  \hat{R}(\tau_1, \tau_2) \Big],  
\vspace{-5mm}    
\end{equation}
\end{small}

where \begin{small}$\hat{R}(\tau_1,\tau_2) =\sum A^{\pi_{t-1}}(s_i, a_i) \frac{ \pi_\theta(a_i|s_i) }{ \pi_{t-1}(a_i|s_i) } + \widehat{\mathrm{WD}}_\gamma( \mathbb{P}_{\pi_{t-1}}^\Phi, \mathbb{P}_{\pi_{\theta}}^\Phi ) $\end{small}.
To optimize the Wasserstein distance we use Algorithm \ref{Alg:wass_constinuous_SGD}. Importantly, stochastic gradients of $F(\theta)$ can be approximated by samples from $\pi_\theta$. In its simplest form, the gradient $\hat{\nabla}_\theta \hat{F}$ can be computed by the vanilla policy gradient over the advantage component and using the REINFORCE estimator through the components involving Test Functions acting on trajectories from $\mathbb{P}_{\pi_\theta}$. 
For \textbf{off-policy} embeddings, $\hat{\nabla}_\theta \hat{F}$ can be computed by sampling from the product of the state probing distributions. Gradients through the differentiable test functions can be computed by the chain rule: $\nabla_\theta \lambda(S, \pi_\theta(S)) = (\nabla_{\phi} \lambda(\phi) )^\top \nabla_\theta \phi $ for $\phi = (S, \pi_\theta(S))$.


BGPG can be thought of as a variant of Trust Region Policy Optimization with a Wasserstein penalty. As opposed to vanilla TRPO, the optimization path of BGPG flows through policy parameter space while encouraging it to follow a smooth trajectory through the geometry of the behavioral manifold. We proceed to show that given the right embedding and cost function, we can prove a monotonic improvement theorem for BGPG, showing that our methods satisfy at least similar guarantees as TRPO.  

Furthermore, 
Let  $V(\pi)$ be the expected reward of policy $\pi$
and $\rho_{\pi}(s) = \mathbb{E}_{\tau \sim \mathbb{P}_{\pi} } \left[    \sum_{t=0}^T \mathbf{1}(s_t = s)     \right]$ be the visitation measure.

Two distinct policies $\pi$ and $\tilde{\pi}$ can be related via the equation (see: \cite{sutton1998introduction}) $V(\tilde{\pi}) = V(\pi)  + \int_\mathcal{S} \rho_{\tilde{\pi}}(s) \left( \int_{\mathcal{A}} \tilde{\pi}(a | s) A^\pi(s,a) da  \right) ds$ and the linear approximations to $V$ around $\pi$ via: $L(\tilde{\pi}) = V(\pi) +   \int_\mathcal{S} \rho_{\pi}(s) \left( \int_{\mathcal{A}} \tilde{\pi}(a | s) A^\pi(s,a) da  \right) ds$ (see: \cite{kakade2002approximately}). 
Let $\mathcal{S}$ be a finite set. Consider the following embedding $\Phi^s : \Gamma \rightarrow \mathbb{R}^{|\mathcal{S}|}$ defined by $\left( \Phi(\tau ) \right)_s=    \sum_{t=0}^T \mathbf{1}(s_t = s)$ and related cost function defined as: $C(\mathbf{v},\mathbf{w}) = \|\mathbf{v}-\mathbf{w}\|_1$. Then $\mathrm{WD}_0( \mathbb{P}_{\tilde{\pi}}^{\Phi^s}, \mathbb{P}_{\pi}^{\Phi^s})$  is related to visitation frequencies since $\mathrm{WD}_0(  \mathbb{P}_{\tilde{\pi}}^{\Phi^s}, \mathbb{P}_{\pi}^{\Phi^s})\geq \sum_{s \in \mathcal{S}} |  \rho_\pi(s) - \rho_{\tilde{\pi}}(s) | $. These observations enable us to prove an analogue of Theorem 1 from \cite{trpo} (see Section \ref{section::appendix_wasserstein_trust_region}
for the proof), namely:
\begin{theorem}\label{theorem::policy_improvement_wasserstein_0}
If $ \mathrm{WD}_0(  \mathbb{P}_{\tilde{\pi}}^{\Phi^s}, \mathbb{P}_{\pi}^{\Phi^s}) \leq \delta$ and $\epsilon  = \max_{s,a} | A^\pi(s,a)  |$, then $V(\tilde{\pi}) \geq L(\tilde{\theta}) - \delta \epsilon$.  
\end{theorem}
As in \cite{trpo}, Theorem \ref{theorem::policy_improvement_wasserstein_0} implies a policy improvement guarantee for BGPG.

\subsection{Behavior Guided Evolution Strategies (BGES)}
\label{sec:bges}

ES takes a black-box optimization approach to RL, by considering a rollout of a policy, parameterized by $\theta$ as a black-box function $F$. This approach has gained in popularity recently \citep{ES, horia, rbo2019}. 
\vspace{-3mm}

\begin{algorithm}[H]
    \textbf{Input:} learning rate $\eta$, noise standard deviation $\sigma$, iterations $T$, BEM $\Phi$, $\beta$ ($>0$ for repulsion, $<0$ for imitation). \\
    \textbf{Initialize:} Initial policy $\pi_0$ parametrized by $\theta_0$, Behavioral Test Functions $\lambda_1, \lambda_2$. Evaluate policy $\pi_0$ to return trajectory $\tau_0$\; \\
    \For{$t= 1, \ldots, T-1$}{
      1. Sample $\epsilon_1, \cdots, \epsilon_n$ independently from $\mathcal{N}(0,I)$. \;\\
      2. Evaluate policies $\{\pi_{t}^k\}_{k=1}^n$ parameterized by $\{\theta_{t} + \sigma \epsilon_k\}_{k=1}^n$, get rewards $R_k$ and trajectories $\tau_k$ for all $k$. \; \\
      3. Update $\lambda_1$ and $\lambda_2$ using Algorithm \ref{Alg:wass_constinuous_SGD}.\; \\
      4. Approximate $\widehat{\mathrm{WD}}\gamma(\mathbb{P}_{\pi_t^k}^\Phi ,\mathbb{P}^{\Phi}_{\pi_t})$ plugging in $\lambda_1, \lambda_2$ into Eq. \ref{equation::empirical_smooth_wass} for each perturbed policy $\pi_k$ \; \\
      5. Update Policy: $\theta_{t+1} = \theta_{t} + \eta \nabla_{ES} F$, where: \; \\
      \vspace{-.7cm}
      \begin{small}
       \begin{align*}
          \nabla_{ES} F = \frac{1}{\sigma}\sum_{k=1}^n [(1-\beta)(R_k - R_t) + \beta \widehat{\mathrm{WD}}_\gamma(\mathbb{P}_{\pi_t^k}^\Phi ,\mathbb{P}^{\Phi}_{\pi_t})]\epsilon_k
      \end{align*}
     \end{small}
      \vspace{-2mm}
     }
     \caption{Behavior-Guided Evolution Strategies}
    \label{Alg:bges}
    \end{algorithm}
\vspace{-5mm}

If we take this approach to optimizing the objective in Eq.~\ref{eq:rlobj}, the result is a black-box optimization algorithm which seeks to maximize the reward and simultaneously maximizes or minimizes the difference in behavior from the base embedding distribution $\mathbb{P}_\mathrm{b}^\Phi$. We call it Behavior-Guided Evolution Strategies (BGES) algorithm (see: Alg. \ref{Alg:bges}).

When $\beta > 0$, and we take $\mathbb{P}_\mathrm{b}^\Phi = \mathbb{P}_{\pi_{t-1}}^\Phi$, BGES resembles the NSR-ES algorithm from \cite{conti}, an instantiation of \textit{novelty search} \citep{Lehman08}. The positive weight on the WD-term enforces newly constructed  policies to be behaviorally different from the previous ones while the $\mathcal{R}-$term drives the optimization to maximize the reward. The key difference in our approach is the probabilistic embedding map, with WD rather than Euclidean distance. We show in Section \ref{sec:exp_maxmax} that BGES outperforms NSR-ES for challenging exploration tasks.

\section{Related Work} 
\label{sec:related}

Our work is related to research in multiple areas in neuroevolution and machine learning:
\vspace{-3mm}
\paragraph{Behavior Characterizations:} The idea of directly optimizing for behavioral diversity was introduced by \cite{Lehman08} and \cite{lehman:thesis}, who proposed to search directly for \textit{novelty}, rather than simply assuming it would naturally arise in the process of optimizing an objective function. This approach has been applied to deep RL \citep{conti} and meta-learning \citep{evolvabilityES}. In all of this work, the policy is represented via a behavioral characterization (BC), which requires domain knowledge. In our setting, we move from deterministic BCs to stochastic behavioral embeddings, thus requiring the use of metrics capable of comparing probabilistic distributions. 
\vspace{-3mm}

\paragraph{Distance Metrics:} WDs have been used in many applications in machine learning where  guarantees based on distributional similarity are required~\citep{jiang2019wasserstein,wgan}. We make use of WDs in our setting for a variety of reasons. First and foremost, the dual formulation of the WD allows us to recover Behavioral Test Functions, providing us with behavior-driven trajectory scores. In contrast to KL divergences, WDs are sensitive to user-defined costs between pairs of samples instead of relying only on likelihood ratios.  Furthermore, as opposed to KL divergences, it is possible to take SGD steps using entropy-regularized Wasserstein objectives. Computing an estimator of the KL divergence is hard without a density model. Since in our framework multiple unknown trajectories may map to the same behavioral embedding, the likelihood ratio between two embedding distributions may be ill-defined.    
\vspace{-3mm}

\paragraph{WDs for RL:} We are not the first to propose using WDs in RL. \cite{zhang} have recently introduced Wasserstein Gradient Flows (WGFs), which casts policy optimization as gradient descent flow on the manifold of corresponding probability measures, where geodesic lengths are given as second-order WDs. We note that computing WGFs is a nontrivial task. In \cite{zhang} this is done via particle approximation methods, which we show in Section \ref{sec:exp} is substantially slower than our methods. The WD has also been employed to replace KL terms in standard Trust Region Policy Optimization~\citep{maginnis}. This is a very special case of our more generic framework (cf.\ Section \ref{sec:algorithms}). In \cite{maginnis} it is suggested to solve the corresponding RL problems via Fokker-Planck equations and diffusion processes, yet no empirical evidence of the feasibility of this approach is provided. We propose general practical algorithms and provide extensive empirical evaluation.
\vspace{-4mm}

\paragraph{Distributional RL} Distributional RL (DRL, \cite{Bellemare2017}) expands on traditional off-policy methods \citep{dqn2013} by attempting to learn a distribution of the return from a given state, rather than just the expected value. These approaches have impressive experimental results \citep{Bellemare2017, dabney18a}, with a growing body of theory \citep{rowland18a, qu19b, bellemare19a, rowland19statistics}. Superficially it may seem that learning a distribution of returns is similar to our approach to PPEs, when the BEM is a distribution over rewards. Indeed, reward-driven embeddings used in DRL can be thought of as special cases of the general class of BEMs. We note two key differences: 1) DRL methods are off-policy whereas our BGES and BGPG algorithms are on-policy, and 2) DRL is typically designed for discrete domains, since Q-Learning with continuous action spaces is generally much harder. Furthermore, we note that while the WD is used in DRL, it is only for the convergence analysis of the DRL algorithm \citep{Bellemare2017}.
\vspace{-2mm}
\section{Experiments}
\label{sec:exp}

Here we seek to test whether our approach to RL translates to performance gains for by evaluating BGPG and BGES, versus their respective baselines for a range of tasks. For each subsection we provide additional details in the Appendix.
\vspace{-2mm}

\subsection{Behavior-Guided Policy Gradient}
\label{sec:exp_minmax}

Our key question is whether our techniques lead to outperformance for BGPG vs. baseline TRPO methods using KL divergence, which are widely used in the reinforcement learning community. For the BEM, we use the concatenation-of-actions, as used already in TRPO. We consider a variety of challenging problems from the DeepMind Control Suite \citep{tassa2018deepmind} and Roboschool (RS). In Fig. \ref{figure:trpo_fullwassersteinupdate} we see that BGPG does indeed outperform KL-based TRPO methods, with gains across all six environments. We also confirm results from \citep{trpo} that a trust region typically improves performance. 
\vspace{-3mm}

\begin{figure}[H]
    \centering\begin{minipage}{0.95\textwidth}
    \subfigure[\textbf{HalfCheetah}]{\includegraphics[width=.3\textwidth]{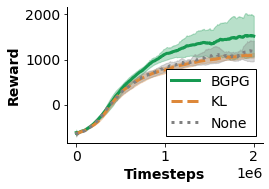}}
    \subfigure[\textbf{Ant}]{\includegraphics[width=.3\textwidth]{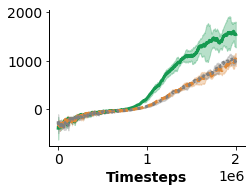}}
   \subfigure[\textbf{Hopper: Hop}]{\includegraphics[width=.3\textwidth]{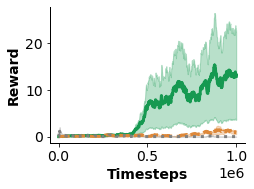}}
    \end{minipage}
    \centering\begin{minipage}{0.95\textwidth}
   \subfigure[\textbf{RS: HalfCheetah}]{\includegraphics[width=.3\textwidth]{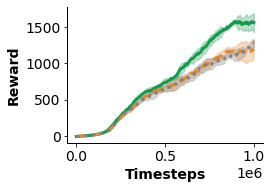}}
    \subfigure[\textbf{Walker: Stand}]{\includegraphics[width=.3\textwidth]{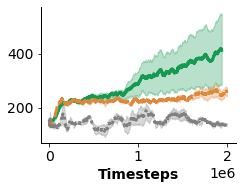}}
    \subfigure[\textbf{RS: Walker2d}]{\includegraphics[width=.3\textwidth]{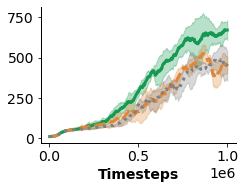}}
    \end{minipage}
    \vspace{-2mm}
    \caption{\small{\textbf{BGPG vs. TRPO:} We compare BGPG and TRPO (KL divergence) on several continuous control tasks. As a baseline we also include results without a trust region ($\beta = 0$ in Algorithm \ref{Alg:bgpg}). Plots show the $\text{mean} \pm \text{std}$ across 5 random seeds.}}
    \label{figure:trpo_fullwassersteinupdate}
\vspace{-3mm}    
\end{figure}

\paragraph{Wall Clock Time:} To illustrate computational benefits of alternating optimization of WD in BGPG, we compare it to the method introduced in \cite{zhang}. In practice, the WD across different state samples can be optimized in a batched manner, details of which are in the Appendix. In Table \ref{table:clocktime} we see that BGPG is substantially faster.

\begin{table}[h]
\caption{\small{Clock time (s) to achieve a normalized reward of $90\%$ of the best achieved. All experiments were run on the same CPU.}}
    \centering
    \scalebox{0.9}{
    \begin{tabular}{l*{3}{c}r}
    \toprule
    & \cite{zhang} & BGPG   \\
    \midrule
    $\mathrm{Pendulum}$ &  3720 & 777 \\
    $\mathrm{Hopper}$: $\mathrm{Stand}$ &  26908 & 10817 \\
    $\mathrm{Hopper}$: $\mathrm{Hop}$ & 23542 & 12820 \\
    $\mathrm{Walker}$: $\mathrm{Stand}$ & 13497 & 4082 \\
    \bottomrule
\end{tabular}}
\end{table}
\label{table:clocktime}

\vspace{-3mm}
\subsection{Behavior-Guided Evolution Strategies}
\label{sec:exp_maxmax}

Next we seek to evaluate the ability for BGES to use its behavioral repulsion for exploration.

\vspace{-3mm}
\paragraph{Deceptive Rewards} A common challenge in RL is \textit{deceptive} rewards. These arise since agents can only learn from data gathered via experience in the environment. To test BGES in this setting, we created two intentionally deceptive environments. In both cases the agent is penalized at each time step for its distance from a goal. The deception comes from a barrier, which means initially positive rewards from moving directly forward will lead to a suboptimal policy. 

We consider two agents---a two-dimensional $\mathrm{point}$ and a larger $\mathrm{quadruped}$. Details are provided in the Appendix (Section \ref{app:exp}). We compare with state-of-the-art \textbf{on-policy} methods for exploration: 
$\mathrm{NSR}$-$\mathrm{ES}$ \citep{conti}, which assumes the BEM is deterministic and uses the Euclidean distance to compare policies, and $\mathrm{NoisyNet}$-$\mathrm{TRPO}$ \cite{noisynet18}. We used the reward-to-go and final state BEMs for the $\mathrm{quadruped}$ and $\mathrm{point}$ respectively.

\begin{figure}[H]
    \begin{minipage}{0.95\textwidth}
	\centering\subfigure[\textbf{Quadruped}]{\includegraphics[width=0.3\textwidth]{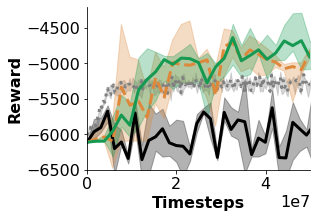}}
	\centering\subfigure[\textbf{Point}]{\includegraphics[width=0.29\textwidth]{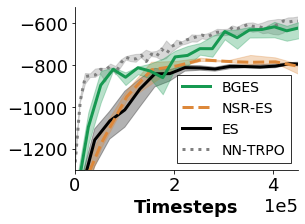}}
	\end{minipage}
	\caption{\small{\textbf{Deceptive Rewards.} Plots show the $\text{mean} \pm \text{std}$ across 5 random seeds for two environments: $\mathrm{Quadruped}$ and $\mathrm{Point}$.}}
	\label{fig:point} 
\end{figure}

Policies avoiding the wall correspond to rewards: $R>-5000$ and $R>-800$ for the $\mathrm{quadruped}$ and $\mathrm{point}$ respectively. In the prior case an agent needs to first learn how to walk and the presence of the wall is enough to prohibit vanilla ES from even learning forward locomotion. As we see in Fig. \ref{fig:point}, BGES is the only method that drives the agent to the goal in \textit{both} settings.

\vspace{-3mm}
\section{Conclusion and Future Work}
\label{sec:con}

In this paper we proposed a new paradigm for on-policy learning in RL, where policies are embedded into expressive latent behavioral spaces and the optimization is conducted
by utilizing the repelling/attraction signals in the corresponding probabilistic distribution spaces. The use of Wasserstein distances (WDs) guarantees flexibility in choosing cost funtions between embedded policy trajectories, enables stochastic gradient steps through corresponding regularized objectives (as opposed to KL divergence methods) and provides an elegant method, via their dual formulations, to quantify behaviorial difference of policies through the behavioral test functions. Furthermore, the dual formulations give rise to efficient algorithms optimizing RL objectives regularized with WDs. 

We also believe the presented methods shed new light on several other challenging problems of modern RL, including: learning with safety guarantees (a repelling signal can be used to enforce behaviors away from dangerous ones) or anomaly detection for reinforcement learning agents (via the above score functions). Finally, we are interested in extending our method to the off policy setting.

\bibliographystyle{abbrv}
\bibliography{wasserstein}

\appendix

\onecolumn

\section*{Appendix: Behavior-Guided Reinforcement Learning}

\section{Additional Experiments}

\begin{figure}[H]
\vspace{-3mm} 
    \begin{minipage}{0.99\textwidth}
	\centering\subfigure{\includegraphics[width=0.39\textwidth]{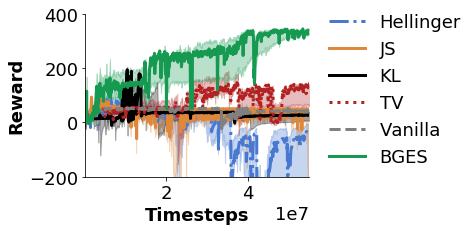}}
	\end{minipage}
	\caption{\small{\textbf{Escaping Local Maxima} A comparison of $\mathrm{BGES}$ with those using different distances on Policy Embeddings.}}
	\label{fig:escape}
\vspace{-2mm} 	
\end{figure}

\vspace{-1mm}
\subsection{Escaping Local Maxima.} 

In Fig. \ref{fig:escape} we compare our methods with methods using regularizers based on other distances or divergences (specifically, Hellinger, Jensen-Shannon (JS), KL and Total Variation (TV) distances), as well as vanilla ES (i.e., with no distance regularizer). 
Experiments were performed on a \textbf{Swimmer} environment from $\mathrm{OpenAI}$ $\mathrm{Gym}$ \citep{brockman2016openai}, where the number of samples of the ES optimizer was drastically reduced. BGES is the only one that manages to obtain good policies which also proves that the benefits come here not just from introducing the regularizer, but from its particular form.

\vspace{-2mm}

\subsection{Imitation Learning}
\label{imitation_exp}

\begin{figure}[H]
\vspace{-3mm} 
    \begin{minipage}{0.99\textwidth}
	\centering\subfigure{\includegraphics[width=0.39\textwidth]{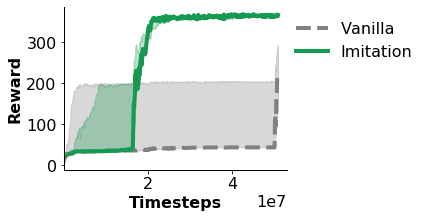}}
	\end{minipage}
    \caption{\small{\textbf{Imitation Learning.}}}
	\label{fig:imit}
\vspace{-2mm} 	
\end{figure}

As discussed in Section \ref{sec:bgpg}, we can also utilize the BGES algorithm for imitation learning, by setting $\beta < 0$, and using an expert's trajectories for the Policy Embedding. For this experiment we use the reward-to-go BEM (Section \ref{sec:wrl}). In Fig.~\ref{fig:imit}, we show that this approach significantly outperforms vanilla ES on the $\mathrm{Swimmer}$ task. Although conceptually simple, we believe this could be a powerful approach with potential extensions, for example in designing safer algorithms. 
\vspace{-2mm}

\newpage
\section{Further Experimental Details}
\label{app:exp}

\subsection{BGPG}
Here we reproduce a full version of Algorithm \ref{Alg:bgpg}:

\begin{algorithm}[H]
\textbf{Input: } Initialize stochastic policy $\pi_0$ parametrized by $\theta_0$, $\beta<0, \eta > 0$, $M,L\in \mathbb{N}$ \\
\For{$t = 1, \dots , T$ }{
   1. Run $\pi_{t-1}$ in the environment to get advantage values $A^{\pi_{t-1}}(s,a)$ and trajectories $\{\tau_i^{(t)} \}_{i=1}^{M}$ \\
   2. Update policy and test functions via several alternating gradient steps over the objective:
  \begin{small}
  \begin{align*}
    F(\theta) &= \mathop{\mathbb{E}}_{\tau_1, \tau_2  \sim  \mathbb{P}_{\pi_{{t-1}}} \bigotimes \mathbb{P}_{\pi_\theta}     }\Big[ \sum_{i=1}^H  A^{\pi_{t-1}}(s_i, a_i) \frac{ \pi_\theta(a_i|s_i) }{ \pi_{t-1}(a_i|s_i) }    \\
    &+  \beta \lambda_1(\Phi(  \tau_1 )) 
    - \beta \lambda_2(\Phi(\tau_2) ) + \beta \gamma \exp\left( \frac{\lambda_1(\Phi(\tau_1)) - \lambda_2(\Phi(\tau_2)) - C(\Phi(\tau_1)), \Phi(\tau_2))}{\gamma}      \right)  \Big]  
  \end{align*} \end{small}\;\\
 Where $\tau_1 =  s_0, a_0, r_0, \cdots, s_H, a_H, r_H $. Let $\theta_{t-1}^{(0)}= \theta_{t-1}$. \; \\
 \For{ $\ell = 1, \cdots , L$ }{
 a. Approximate $\mathbb{P}_{\pi_{{t-1}}} \bigotimes \mathbb{P}_{\pi_\theta}$ via $\frac{1}{M} \{  \tau_i^{(t)}\}_{i=1}^M\bigotimes \frac{1}{M} \{ \tau_i^\theta \}_{i=1}^M:= \hat{P}_{\pi_t, \pi_\theta}$ where $\tau_i^\theta \stackrel{\text{i.i.d}}{\sim}\mathbb{P}_{\pi_\theta}$\;\\
 b. Take SGA step $\theta_{t-1}^{(\ell)} = \theta_{t-1}^{(\ell - 1)} + \eta \hat{\nabla}_\theta \hat{F}(\theta_{t-1}^{(\ell -1)})$  using samples from $\hat{P}_{\pi_{t-1}, \pi_\theta}$. \;\\
  c. Use samples from $\hat{P}_{\pi_{t-1}, \pi_\theta}$ and Algorithm \ref{Alg:wass_constinuous_SGD}  to update $\lambda_1, \lambda_2$. \;\\
  }
}
Set $\theta_t = \theta_{t-1}^{(M)}$.
 \caption{Behvaior-Guided Policy Gradient with \textbf{On-Policy} Embeddings}
\label{Alg:bgpg_full}
\end{algorithm}

\begin{algorithm}[H]
\textbf{Input: } Initialize stochastic policy $\pi_0$ parametrized by $\theta_0$, $\beta<0, \eta > 0$, $M,L\in \mathbb{N}$, state probing distribution $\mathbb{P}_{\mathcal{S}_0}$. \\
\For{$t = 1, \dots , T$ }{
   1. Run $\pi_{t-1}$ in the environment to get advantage values $A^{\pi_{t-1}}(s,a)$ .
   2. Update policy and test functions via several alternating gradient steps over the objective:
  \begin{align*}
    F(\theta) &= \mathop{\mathbb{E}}_{\tau_1 \sim  \mathbb{P}_{\pi_{{t-1}}}    }\Big[ \sum_{i=1}^H  A^{\pi_{t-1}}(s_i, a_i) \frac{ \pi_\theta(a_i|s_i) }{ \pi_{t-1}(a_i|s_i) } \Big]    \\
    &+ \mathbb{E}_{s_1, s_2 \sim \mathbb{P}_{\mathcal{S}_\theta} \bigotimes \mathbb{P}_{\mathcal{S}_{t-1}}}\Big[ \beta \lambda_1(s_1, \pi_\theta(s_1) ) - \beta \lambda_2( s_2, \pi_{t-1}(s_2) ) \\
    &+\beta \gamma \exp\left( \frac{\lambda_1(s_1, \pi_\theta(s_1)) - \lambda_2(s_2, \pi_{t-1}(s_2)) - C((s_1, \pi_{\theta}(s_1)), (s_2, \pi_{t-1}(s_2)))}{\gamma}      \right)  \Big]  
  \end{align*} \;\\
 Where $\tau_1 =  s_0, a_0, r_0, \cdots, s_H, a_H, r_H $. Let $\theta_{t-1}^{(0)}= \theta_{t-1}$. \; \\
 \For{ $\ell = 1, \cdots , L$ }{
 a. Approximate the expectation $\mathbb{E}_{s_1, s_2 \sim \mathbb{P}_{\mathcal{S}_\theta} \bigotimes \mathbb{P}_{\mathcal{S}_{t-1}}}$ via $2M$ samples. \;\\
 b. Take SGA step $\theta_{t-1}^{(\ell)} = \theta_{t-1}^{(\ell - 1)} + \eta \hat{\nabla}_\theta \hat{F}(\theta_{t-1}^{(\ell -1)})$  using samples from a. and trajectories from current $\pi_\theta$. \;\\
  c. Use samples from $a.$ and Algorithm \ref{Alg:wass_constinuous_SGD}  to update $\lambda_1, \lambda_2$. \;\\
  }
}
Set $\theta_t = \theta_{t-1}^{(M)}$.
 \caption{Behvaior-Guided Policy Gradient with \textbf{Off-Policy} Embeddings}
\label{Alg:bgpg_full_off_policy}
\end{algorithm}

\paragraph{A Lower-variance Gradient Estimator via \textbf{Off-Policy} embeddings:} As explained in Section 5.2, the BGPG considers an objective which involves two parts: the conventional surrogate loss function for policy optimization \citep{schulman2017proximal}, and a loss function that involves the Behavior Test Functions. Though we could apply vanilla reinforced gradients on both parts, it is straightforward to notice that the second part can be optimized with reparameterized gradients \citep{kingma2013auto}, which arguably have lower variance compared to the reinforced gradients. In particular, we note that under random feature approximation (\ref{equation::empirical_smooth_wass}), as well as the action-concatenation embedding, the Wasserstein distance loss $\widehat{\text{WD}}_\gamma(P_{\pi_\theta}^\Phi,P_b^\Phi)$ is a differentiable function of $\theta$. To see this more clearly, notice that under a Gaussian policy $a \sim \mathcal{N}(\mu_\theta(s), \sigma_\theta(s)^2)$  the actions $a = \mu_\theta(s) + \sigma_\theta(s) \cdot \epsilon$ are reparametrizable for $\epsilon$ being standard Gaussian noises. We can directly apply the reparametrization trick to this second objective to obtain a gradient estimator with potentially much lower variance. In our experiments, we applied this lower-variance gradient estimator. In Algorithm \ref{Alg:bgpg_full_off_policy} we allow the state probing distribution to evolve with the iteration index of the algorithm $t$.

\paragraph{Trust Region Policy Optimization:} Though the original TRPO \citep{trpo} construct the trust region based on KL-divergence, we propose to construct the trust region with WD. For convenience, we adopt a dual formulation of the trust region method and aim to optimize the augmented objective $\mathbb{E}_{\tau\sim\pi_\theta}[R(\tau)] - \beta \text{WD}_\gamma (\mathbb{P}_{\pi^\prime}^\Phi,\mathbb{P}_{\pi_\theta}^\Phi)$. We apply the concatenation-of-actions embedding and random feature maps to calculate the trust region. We identify several important hyperparameters: the RKHS (for the test function) is produced by RBF kernel $k(x,y) = \exp(\Vert x-y\Vert_2^2 / \sigma^2)$  with $\sigma = 0.1$; the number of random features is $D=100$; recall the embedding is $\Phi(\tau) = [a_1,a_2...a_H]$ where $H$ is the horizon of the trajectory, here we take $10$ actions per state and embed them together, this is equivalent to reducing the variance of the gradient estimator by increasing the sample size; the regularized entropy coefficient in the WD definition as $\gamma = 0.1$; the trust region trade-off constant $\beta \in \{0.1,1,10\}$. The alternate gradient descent is carried out with $T=100$ alternating steps and test function coefficients $\mathbf{p}\in \mathbb{R}^D$ are updated with learning rate $\alpha_\mathbf{p} = 0.01$. 

The baseline algorithms are: No trust region, and trust region with KL-divergence. The KL-divergence is identified by a maximum KL-divergence threshold per update, which we set to $\epsilon = 0.01$.

Across all algorithms, we adopt the open source implementation \citep{baselines}. Hyper-parameters such as number of time steps per update as well as implementation techniques such as state normalization are default in the original code base.

The additional experiment results can be found in Figure \ref{figure:trpoappendix} where we show comparison on additional continuous control benchmarks: Tasks with DM are from DeepMind Contol Suites \citep{tassa2018deepmind}. We see that the trust region constructed from the WD consistently outperforms other baselines (importantly, trust region methods are always better than the baseline without trust region, this confirms that trust region methods are critical in stabilizing the updates).

\begin{figure}[H]
\centering
    \centering\subfigure[\textbf{Reacher}]{\includegraphics[width=.24\linewidth]{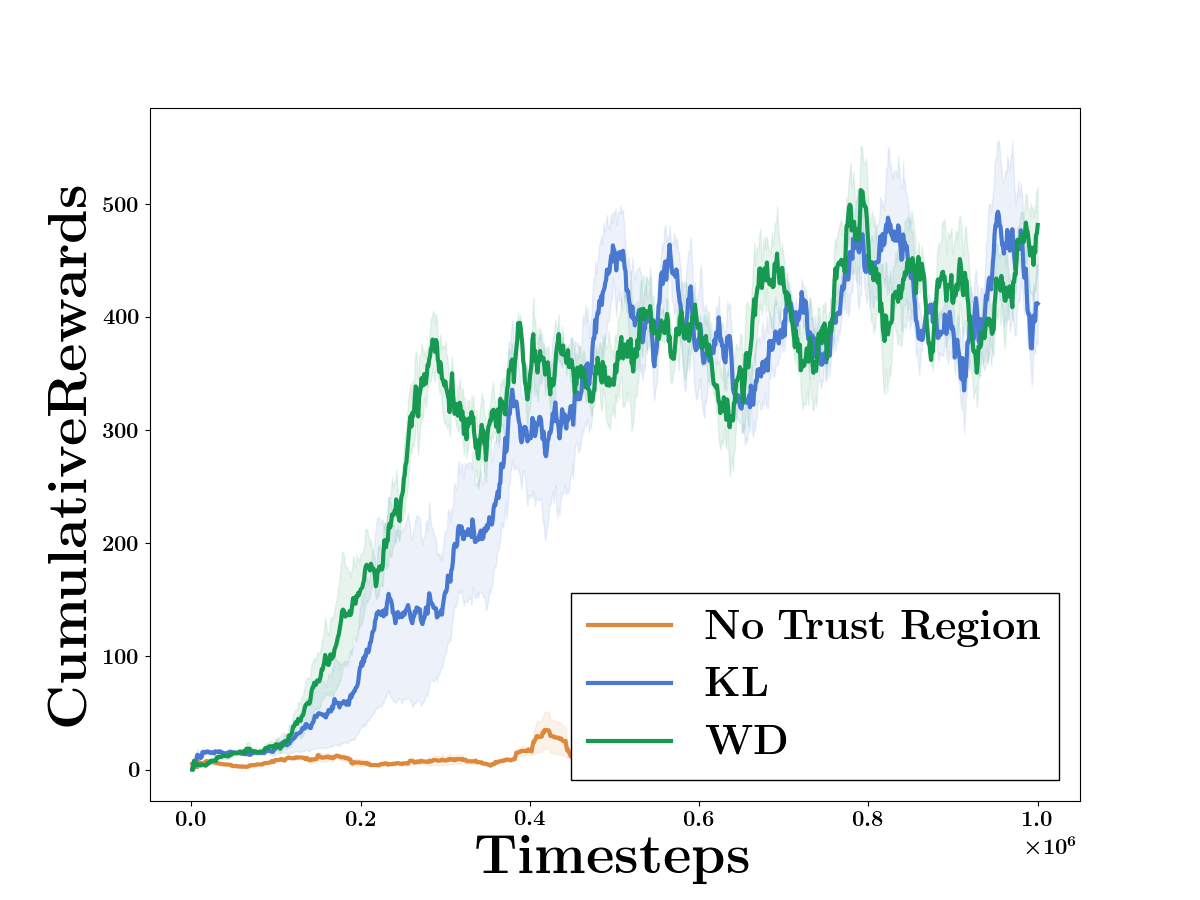}}
    \centering\subfigure[\textbf{MountainCar}]{\includegraphics[width=.24\linewidth]{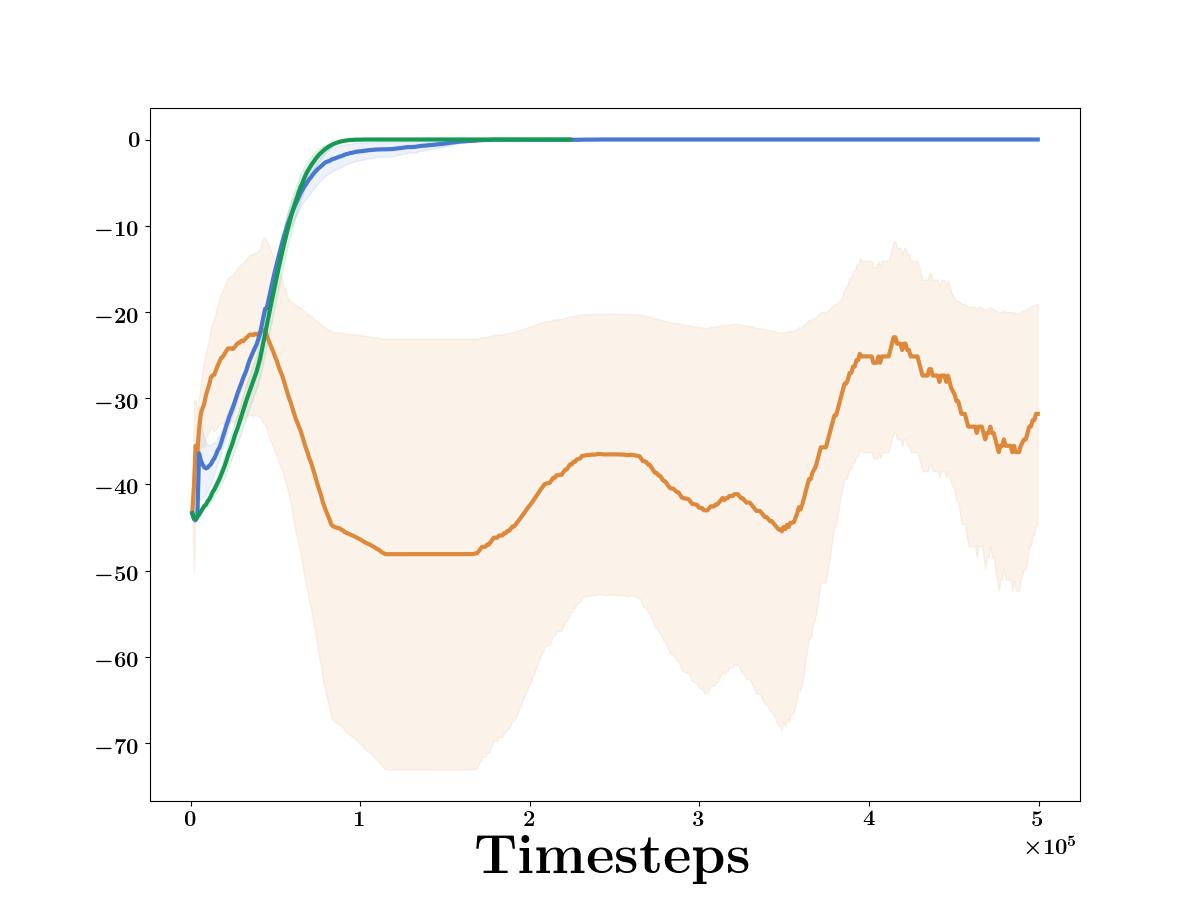}}
    \centering\subfigure[\textbf{Acrobot}]{\includegraphics[width=.24\linewidth]{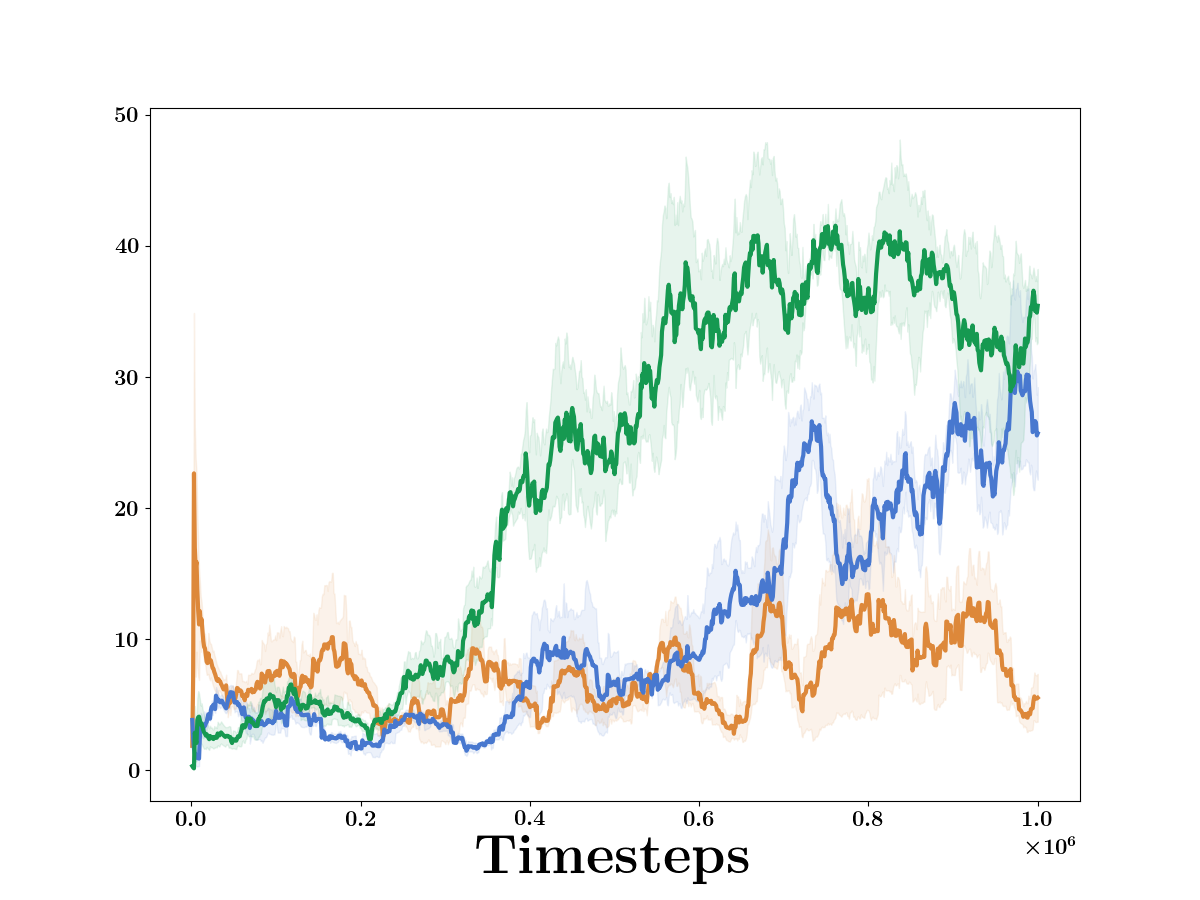}}
    \caption{\small{Additional Experiment on TRPO. We compare No Trust Region with two alternative trust region constructions: KL-divergence and Wassertein distance (ours).}}
    \label{figure:trpoappendix}
\vspace{-3mm}    
\end{figure}

\paragraph{Wasserstein AO vs. Particle Approximation:} To calculate the regularized Wasserstein distance, we propose a gradient descent method that iteratively updates the test function. The alternting optimization (AO) scheme consists of updating both the test function and the distribution parameters such that the regularized Wasserstein distance of the trainable distribution against the reference distribution is minimized. Alternatively, we can also adopt a particle approximation method to calculate the Wasserstein distance and update the distribution parameters using an approximate gradient descent method \citep{zhang}. We see the benefit in clock time in Fig \ref{figure:wall_clock_time}. 

\begin{figure}[H]
    \begin{minipage}{0.99\textwidth}
    \subfigure[\textbf{Pendulum}]{\includegraphics[width=.24\textwidth]{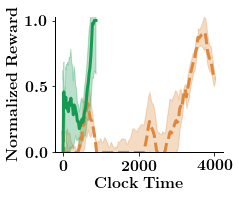}}
    \subfigure[\textbf{Hopper: Stand}]{\includegraphics[width=.24\textwidth]{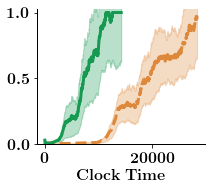}}
    \subfigure[\textbf{Hopper: Hop}]{\includegraphics[width=.24\textwidth]{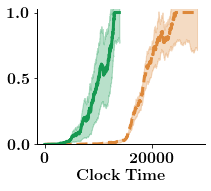}}
    \subfigure[\textbf{Walker: Stand}]{\includegraphics[width=.24\textwidth]{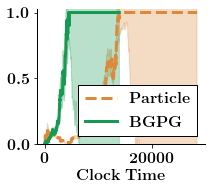}}
    \caption{\small{The clock-time comparison (in sec) of BGPG (alternating optimization) with particle approximation.}}
    \label{figure:wall_clock_time}
    \end{minipage}
\vspace{-3mm}    
\end{figure}

One major advantage of BGPG against particle approximation is its ease of parallelization. In particular, when using the concatenation-of-actions embedding, the aggregate Wasserstein distance can be decomposeed into an average of a set of Wasserstein distances over states. To calculate this aggregated gradient, BGPG can easily leverage the matrix multiplication; on the other hand, particle approximation requires that the dual optimal variables of each subproblem be computed, which is not straightforward to parallelize. 

We test both methods in the context of trust region policy search, in which we explicitly calculate the Wasserstein distance of consecutive policies and enforce the constraints using a line search as in \citep{trpo}. Both methods require the trust region trade-off parameter $\beta \in \{0.1,1,10\}$. We adopt the particle method in \citep{zhang} where for each state there are $M=16$ particles. The gradients are derived based a RKHS where we adaptively adjust the coefficient of the RBF kernel based on the mean distance between particles.
For the AO, we find that it suffices to carry out $T\in \{1,5,10\}$ gradient descents to approximate the regularized Wasserstein distance.

\subsection{BGES}

Here we reproduce a detailed version of Algorithm \ref{Alg:bges}:

\begin{algorithm}[h]
\textbf{Input:} learning rate $\eta$, noise standard deviation $\sigma$, iterations $T$, BEM $\Phi$, $\beta$ \\
\textbf{Initialize:} Initial policy $\pi_0$ parametrized by $\theta_0$, Behavioral Test Functions $\lambda_1, \lambda_2$. Evaluate policy $\pi_0$ to return trajectory $\tau_0$ and subsequently use the BEM to produce an initial $\hat{\mathbb{P}}_{\pi_0}^\Phi$. \; \\
\For{$t= 1, \ldots, T-1$}{
  1. Sample $\epsilon_1, \cdots, \epsilon_n$ independently from $\mathcal{N}(0,I)$. \;\\
  2. Evaluate policies $\{\pi_{t}^k\}_{k=1}^n$ parameterized by $\{\theta_{t} + \sigma \epsilon_k\}_{k=1}^n$ to return rewards $R_k$ and trajectories $\tau_k$ for all $k$. \; \\
  3. Use BEM to map trajectories $\tau_k$ to produce empirical $\hat{\mathbb{P}}_{\pi_t^k}^\Phi$ for all $k$. \; \\
  4. Update $\lambda_1$ and $\lambda_2$ using Algorithm \ref{Alg:wass_constinuous_SGD}, where $\mu = \frac{1}{n}\cup_{k=1}^n \hat{ \mathbb{P}}_{\pi_{t-1}^k}^\Phi$ and $\nu = \frac{1}{n}\cup_{k=1}^n \hat{ \mathbb{P}}_{\pi_t^k}^\Phi $ are the uniform distribution over the set of from 3 for $t-1$ and $t$. \; \\
  5. Approximate $\widehat{\mathrm{WD}}\gamma(\mathbb{P}_{\pi_t^k}^\Phi ,\mathbb{P}^{\Phi}_{\pi_t})$ plugging in $\lambda_1, \lambda_2$ into Eq. \ref{equation::empirical_smooth_wass} for each perturbed policy $\pi_k$ \; \\
  6. Update Policy: $\theta_{t+1} = \theta_{t} + \eta \nabla_{ES} F$, where: \; \\
  \vspace{-4mm}
   \begin{align*}
      \nabla_{ES} F = \frac{1}{\sigma}\sum_{k=1}^n [(1-\beta)(R_k - R_t) + \beta \widehat{\mathrm{WD}}_\gamma(\mathbb{P}_{\pi_t^k}^\Phi ,\mathbb{P}^{\Phi}_{\pi_t})]\epsilon_k
  \end{align*}
  \vspace{-2mm}
 }
 \caption{Behavior-Guided Evolution Strategies with \textbf{On-Policy} Embeddings}
\label{Alg:bges_full}
\end{algorithm}

\paragraph{Efficient Exploration:} To demonstrate the effectiveness of our method in exploring deceptive environments, we constructed two new environments using the MuJoCo simulator. For the point environment, we have a $6$ dimensional state and $2$ dimensional action, with the reward at each timestep calculated as the distance between the agent and the goal. We use a horizon of $50$ which is sufficient to reach the goal. The quadruped environment is based on $\mathrm{Ant}$ from the $\mathrm{Open}$ $\mathrm{AI}$ $\mathrm{Gym}$ \citep{brockman2016openai}, and has a similar reward structure to the point environment but a much larger state space (113) and action space (8). For the quadruped, we use a horizon length of $400$. 

To leverage the trivially parallelizable nature of ES algorithms, we use the $\mathrm{ray}$ library, and distribute the rollouts across $72$ workers using AWS. Since we are sampling from an isotropic Gaussian, we are able to pass only the seed to the workers, as in \cite{ES}. However we do need to return trajectory information to the master worker.

For both the point and quadruped agents, we use random features with dimensionality $m=1000$, and $100$ warm-start updates for the WD at each iteration. For point, we use the final state embedding, learning rate $\eta = 0.1$ and $\sigma$ = $0.01$. For the quadruped, we use the reward-to-go embedding, as we found this was needed to learn locomotion, as well as a learning rate of $\eta = 0.02$ and $\sigma$ = 0.02. The hyper-parameters were the same for all ES algorithms. When computing the WD, we used the previous $2$ policies, $\theta_{t-1}$ and $\theta_{t-2}$. 

Our approach includes several new hyperparameters, such as the kernel for the Behavioral Test Functions and the choice of BEM. For our experiments we did not perform any hyperparameter optimization. We only considered the rbf kernel, and only varied the BEM for BGES. For BGES, we demonstrated several different BEMs, and we show an ablation study for the point agent in Fig. \ref{fig:maxmaxablation} where we see that both the reward-to-go (RTG) and Final State (SF) worked, but the vector of all states (SV) did not (for 5 seeds). We leave learned BEMs as exciting future work.

\begin{figure}[H]
    \begin{minipage}{0.99\textwidth}
	\centering\subfigure[\textbf{Embeddings}]{\includegraphics[width=0.25\textwidth]{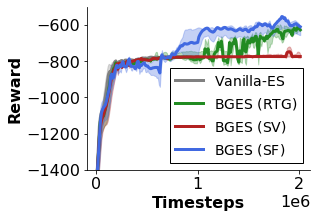}}	
	\centering\subfigure[\textbf{Previous Policies}]{\includegraphics[width=0.25\textwidth]{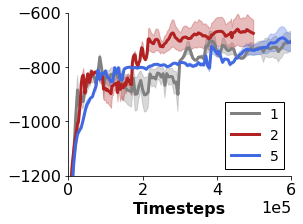}}	
	\end{minipage}
	\caption{A sensitivity analysis investigating a) the impact of the embedding and b) the number of previous policies $\theta_{t-i},  i \in {1, 2, 5}$}
	\label{fig:maxmaxablation} 
\end{figure}

For embeddings, we compare the reward-to-go (RTG), concatenation of states (SV) and final state (SF). In both the RTG and SF case the agent learns to navigate past the wall ($>-800$). For the number of previous policies, we use the SF embedding, and using $2$ appears to work best, but both $1$ and $5$ do learn the correct behavior.

\paragraph{Escaping Local Maxima:} We also demonstrated that our method leads to faster training even in more standard settings, where exploration is not that crucial, but the optimization can be trapped in local maxima. To show it, we compared baseline ES algorithm for ES optimization from \cite{ES} with its enhancements, where regularizers using different metrics on the space of probabilistic distributions corresponding to policy embeddings were used, as in the previous paragraph. We noticed that adding Wasserstein regularizers drastically improved optimization, whereas regularizers based on other distances/divergencies, namely: Hellinger, Jensen-Shannon, KL and TV did not have any impact. We considered $\mathrm{Swimmer}$ task from $\mathrm{OpenAI}$ $\mathrm{Gym}$ and to make it challenging, reduced the number of perturbations per iteration to $80$. In that setting our method was the only one that was not trapped in local maxima and managed to learn effective policies.

\subsection{Imitation Learning:}\label{section::appendix_imitation} For the Imitation Learning experiment we used the reward-to-go embedding, with learning rate $\eta = 0.1$ and $\sigma$ = $0.01$. We use one oracle policy, which achieves $>360$ on the environment. The only information provided to the algorithm is the embedded trajectory, used to compute the WD. This has exciting future applications since no additional information about the oracle is required in order to significantly improve learning.

\subsection{Repulsion Learning and Attraction learning}
\label{Sec:repulsion}

Here we reproduce a full version of Algorithm \ref{Alg:repulsion_learning}:

\begin{algorithm}[H]
\textbf{Input: } $\beta, \eta > 0$, $M\in \mathbb{N}$ \\
\textbf{Initialize:} Initial stochastic policies $\pi_0^\mathbf{a}, \pi_0^\mathbf{b}$, parametrized by $\theta_0^\mathbf{a}, \theta_0^\mathbf{b} $ respectively, Behavioral Test Functions $\lambda_1^\mathbf{a}, \lambda_2^\mathbf{b}$\; \\
\For{$t = 1, \dots , T$ }{
   1. Collect $M$ trajectories $\{ \tau_i^\mathbf{a} \}_{i=1}^M$ from $\mathbb{P}_{\pi_{t-1}^\mathbf{a}}$ and $M$ trajectories $\{ \tau_i^\mathbf{b}  \}_{i=1}^M$ from $\mathbb{P}_{\pi_{t-1}^\mathbf{b}}$. Approximate $\mathbb{P}_{\pi_{t-1}^\mathbf{a}} \bigotimes \mathbb{P}_{\pi_{t-1}^{\mathbf{b}}}$ via $\frac{1}{M} \{ \tau_i^\mathbf{a} \}_{i=1}^M  \bigotimes \frac{1}{M}  \{ \tau_i^\mathbf{b}  \}_{i=1}^M  := \hat{P}_{\pi_{t-1}^\mathbf{a} , \pi_{t-1}^\mathbf{b}}$\\
   2. Form two distinct surrogate rewards for joint trajectories of agents $\mathbf{a}$ and $\mathbf{b}$:\\
  \begin{align*}
  \tilde{R}_\mathbf{a}(\tau_1, \tau_2) &=  \mathcal{R}(\tau_1)  +  \beta \lambda_1^\mathbf{a}(\Phi(  \tau_1  ) )   + 
  \beta \gamma \exp\left( \frac{\lambda_1^\mathbf{a}(\Phi(\tau_1)) - \lambda_2^\mathbf{b}(\Phi(\tau_2)) - C(\Phi(\tau_1)), \Phi(\tau_2))}{\gamma}      \right) \\
  \tilde{R}_\mathbf{b}(\tau_1, \tau_2) &=  \mathcal{R}(\tau_2)    - \beta \lambda_2^\mathbf{b}(\Phi(\tau_2) )+ 
  \beta \gamma \exp\left( \frac{\lambda_1^\mathbf{a}(\Phi(\tau_1)) - \lambda_2^\mathbf{b}(\Phi(\tau_2)) - C(\Phi(\tau_1)), \Phi(\tau_2))}{\gamma}      \right) 
  \end{align*} \;\\
  3. For $\mathbf{c} \in \{\mathbf{a}, \mathbf{b} \}$ use the Reinforce estimator to take gradient steps:
  \begin{align*}
      \theta_t^{\mathbf{c}} = \theta_{t-1}^{\mathbf{c}} + \eta  \mathop{\mathbb{E}}_{\tau^\mathbf{a}, \tau^\mathbf{b} \sim \hat{P}_{\pi_{t-1}^\mathbf{a} , \pi_{t-1}^\mathbf{b}}  }\left[       \tilde{\mathcal{R}}_{\mathbf{c}}(\tau^\mathbf{a}, \tau^{\mathbf{b}})       \left(\sum_{i=0}^{H-1}     \nabla_{\theta_{t-1}^\mathbf{c} }  \log\left(  \pi_{t-1}^{\mathbf{c}}(a_i^\mathbf{c} | s_i^\mathbf{c})    \right)     \right) \right] 
  \end{align*}
  Where $\tau^{\mathbf{a}} =  s_0^{\mathbf{a}}, a_0^{\mathbf{a}}, r_0^{\mathbf{a}}, \cdots, s_H^{\mathbf{a}}, a_H^{\mathbf{a}}, r_H^{\mathbf{a}} $ and $\tau^{\mathbf{b}} =  s_0^{\mathbf{b}}, a_0^\mathbf{b}, r_0^{\mathbf{b}}, \cdots, s_H^\mathbf{b}, a_H^\mathbf{b}, r_H^\mathbf{b} $. \; \\
  5. Use samples from $\hat{P}_{\pi_{t-1}^\mathbf{a} , \pi_{t-1}^\mathbf{b}}$ and Algorithm \ref{Alg:wass_constinuous_SGD} to update the Behavioral Test Functions $\lambda_1^\mathbf{a}, \lambda_2^\mathbf{b}$.
}
 \caption{Behvaior-Guided Repulsion (and Attraction) Learning with \textbf{On-Policy} Embeddings}
\label{Alg:repulsion_learning_full}
\end{algorithm}

Algorithm \ref{Alg:repulsion_learning_full} is the de version of the repulsion algorithm from Section \ref{section::repulsion_learning}. The algorithm maintains two policies $\pi^\mathbf{a}$ and $\pi^{\mathbf{b}}$. Each policy is optimized by taking a policy gradient step (using the Reinforce gradient estimator) in the direction optimizing surrogate rewards $\tilde{\mathcal{R}}_\mathbf{a}$ and $\tilde{\mathcal{R}}_\mathbf{b}$ that combines the signal from the task's reward function $\mathcal{R}$ and the repulsion (or attraction) score encoded by the behavioral test functions $\lambda^{\mathbf{a}}$ and $\lambda^{\mathbf{b}}$. 

We conducted experiments testing Algorithm \ref{Alg:repulsion_learning} on a simple Mujoco environment consisting of a particle that moves on the plane and whose objective is to learn a policy that allows it to reach one of two goals. Each policy outputs a velocity vector and stochasticity is achieved by adding Gaussian noise to the mean velocity encoded by a neural network with two size 5 hidden layers and ReLu activations. If an agent performs action $a$ at state $s$, it moves to state $a+s$. The reward of an agent after performing action $a$ at state $s$  equals $-\| a \|^2*30 - \min(d(s, \text{Goal}_1), d(s, \text{Goal}_2))^2 $ where $d(x,y)$ denotes the distance between $x$ and $y$ in $\mathbb{R}^2$. The initial state is chosen by sampling a Gaussian distribution with mean $\binom{0}{0}$ and diagonal variance $0.1$. In each iteration step we sample $100$ trajectories. In the following pictures we plot the policies' behavior by plotting $100$ trajectories of each. The embedding $\Phi : \Gamma \rightarrow \mathbb{R}$ maps trajectories $\tau$ to their mean displacement in the $x-$axis. We use the squared absolute value difference as the cost function. When $\beta  < 0$ we favour attraction and the agent are encouraged to learn a similar policy to solve the same task.
\begin{figure}[H]
    \begin{minipage}{0.99\textwidth}
	\centering\subfigure[$\pi_0^\mathbf{a}$]{\includegraphics[width=0.20\textwidth]{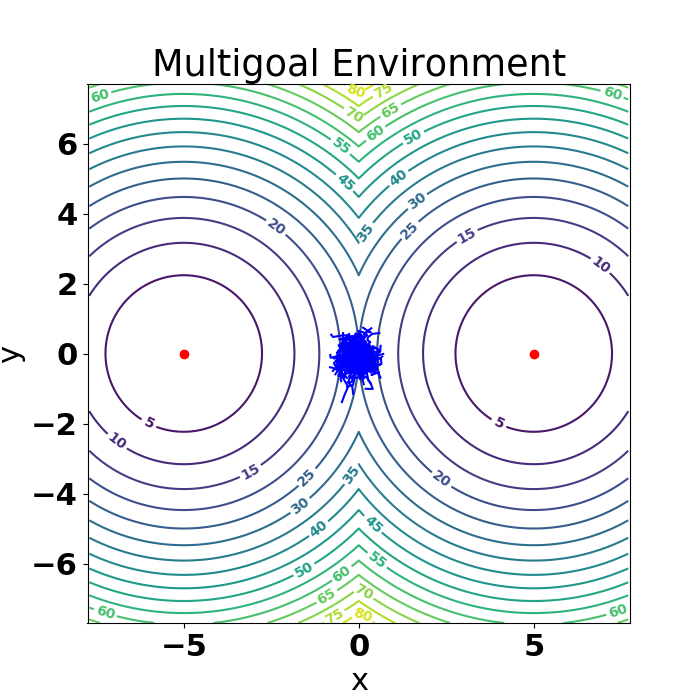}}	
	\centering\subfigure[$\pi_0^\mathbf{b}$]{\includegraphics[width=0.20\textwidth]{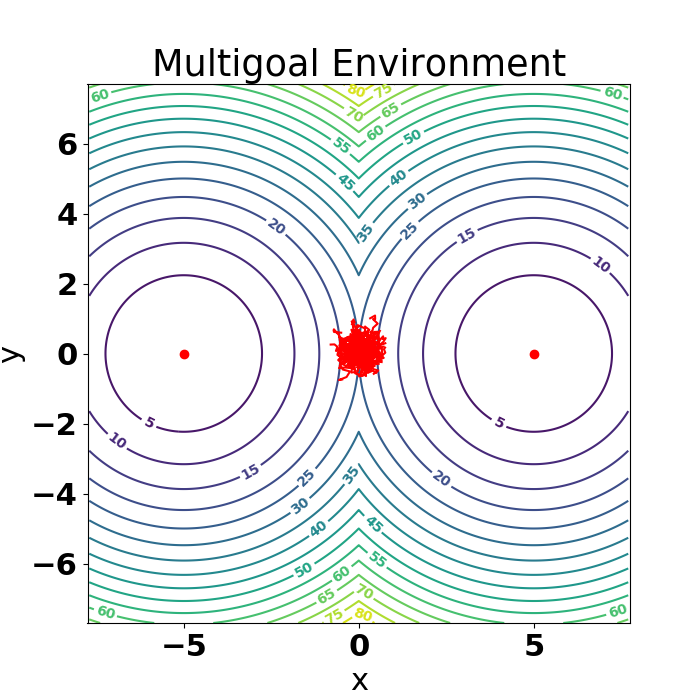}}
	\centering\subfigure[$\lambda^{\mathbf{a}}$ and $-\lambda^{\mathbf{b}}$ at $t=0$]{\includegraphics[width=0.3\textwidth]{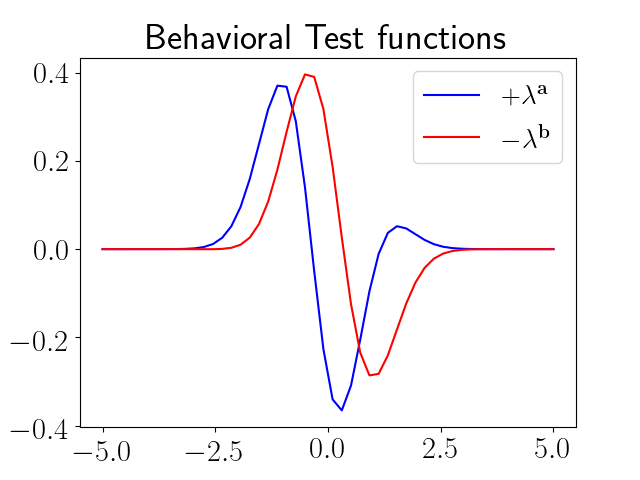}}
	\end{minipage}
	\caption{Initial state of policies $\pi^{\mathbf{a}}, \pi^{\mathbf{b}}$ and Behavioral Test functions $\lambda^{\mathbf{a}}, \lambda^{\mathbf{b}}$ in the Multigoal environment.  }
	\label{fig:maxmaxablation} 
\end{figure}
There are two optimal policies, moving the particle to the left goal or moving it to the right goal. We now plot how the policies' behavior and evolves throughout optimization and how the Behavioral Test Functions guide the optimization by favouring the two policies to be close by or far apart. 
\begin{figure}[H]
    \begin{minipage}{0.99\textwidth}
	\centering\subfigure[$\pi_{22}^\mathbf{a}$]{\includegraphics[width=0.20\textwidth]{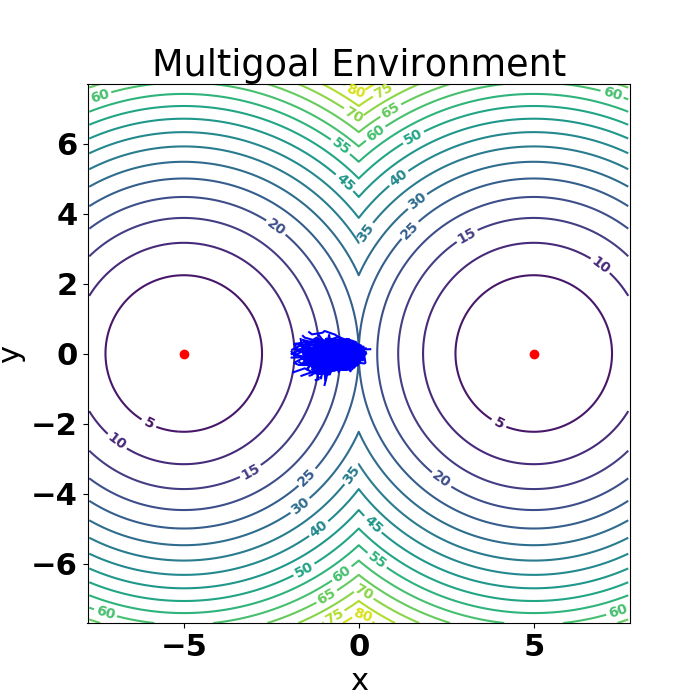}}	
	\centering\subfigure[$\pi_{22}^\mathbf{b}$]{\includegraphics[width=0.20\textwidth]{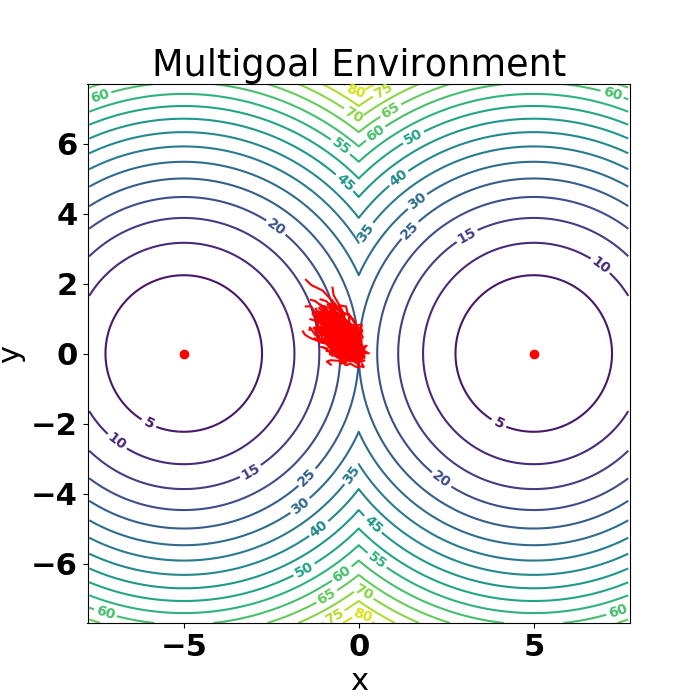}}
	\centering\subfigure[$\lambda^{\mathbf{a}}$ and $-\lambda^{\mathbf{b}}$ at $t=22$]{\includegraphics[width=0.3\textwidth]{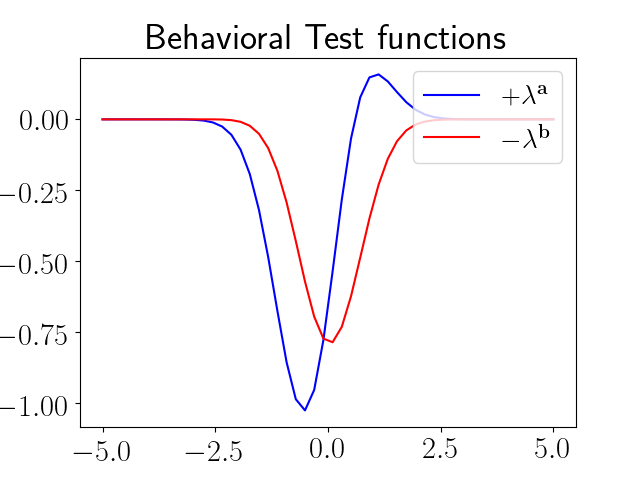}}
	\end{minipage}
	    \begin{minipage}{0.99\textwidth}
	\centering\subfigure[$\pi_{118}^\mathbf{a}$]{\includegraphics[width=0.20\textwidth]{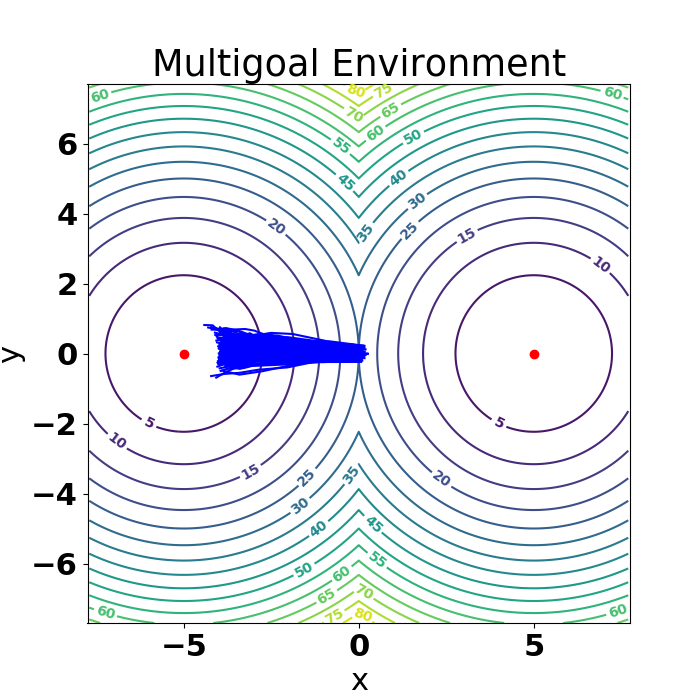}}	
	\centering\subfigure[$\pi_{118}^\mathbf{b}$]{\includegraphics[width=0.20\textwidth]{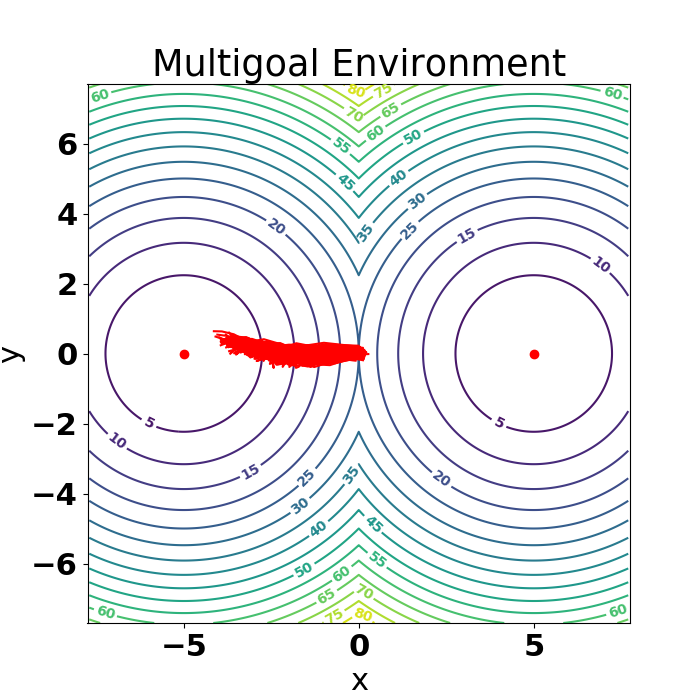}}
	\centering\subfigure[$\lambda^{\mathbf{a}}$ and $-\lambda^{\mathbf{b}}$ at $t=118$]{\includegraphics[width=0.3\textwidth]{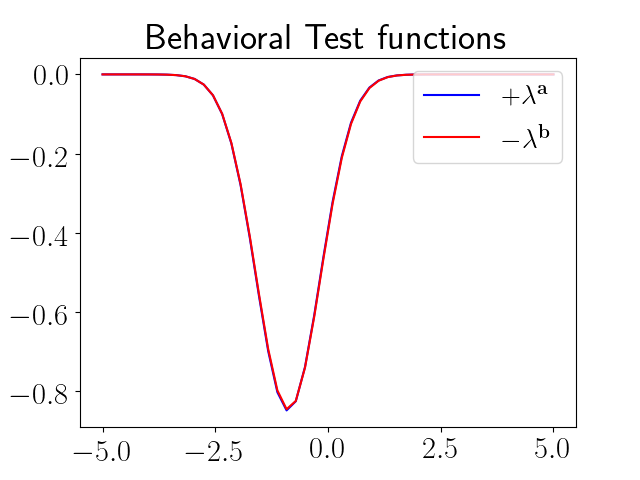}}
	\end{minipage}
	\caption{Evolution of the policies and Behavioral Test Functions throughout optimization.  }
	\label{fig:maxmaxablation} 
\end{figure}

Let $\mathcal{X}$ and $\mathcal{Y}$ be the domains of two measures $\mu$, and $\nu$. Recall that in case $\gamma = 0$, $\mathcal{X} = \mathcal{Y}$, and $C(x,x)= 0$ for all $x \in \mathcal{X}$, then $\lambda^*_\mu(x) = \lambda^*_\nu(x) = \lambda^*(x)$ for all $x \in \mathcal{X}$. In the case of regularized Wasserstein distances with $\gamma >0$, this relationship may not hold true even if the cost satisfies the same diagonal assumption. For example when the regularizing measure is the product measure, and $\mu, \nu$ have disjoint supports, since the soft constraint $\gamma \exp\left( \frac{ \lambda_\mu(\mathbf{x})  - \lambda_\nu(\mathbf{y}) - C(\mathbf{x}, \mathbf{y})}{\gamma}      \right)$ is enforced in expectation over the product measure there may exist optimal solutions $\lambda_\mu^*, \lambda^*_\nu$ that do not satisfy $\lambda^*_\mu = \lambda^*_\nu$.

\section{Theoretical results}
\label{sec:theory}

We start by exploring some properties of the Wasserstein distance and its interaction with some simple classes of embeddings. The first lemma we show has the intention to show conditions under which two policies can be shown to be equal provided the Wasserstein distance between its trajectory embeddings is zero. This result implies that our framework is capable of capturing equality of policies when the embedding space equals the space of trajectories.

\begin{lemma}\label{lemma::policy_equality_simple_appendix}
Let $\mathcal{S}$ and $\mathcal{A}$ be finite sets, the $\mathrm{MDP}$ be episodic (i.e. of finite horizon $H$), and $\Phi(\tau)=\sum_{t=0}^H e_{s_t, a_t}$ with $e_{s,a} \in \mathbb{R}^{|\mathcal{S}| + |\mathcal{A}|}$ the indicator vector for the state action pair $(s,a)$. Let $C(\mathbf{v},\mathbf{w}) = \|\mathbf{v}-\mathbf{w}\|_p^p$ for $p \geq 1$. If $\gamma = 0$ and $\mathrm{WD}_\gamma(\mathbb{P}_{\pi}^\Phi,   \mathbb{P}_{\pi'}^\Phi ) = 0$ then $\pi = \pi'$.
\end{lemma}

\begin{proof}
If $\mathrm{WD}_\gamma(\mathbb{P}_{\pi}^\Phi,   \mathbb{P}_{\pi'}^\Phi ) = 0$, there exists a coupling $\Pi$ between $\mathbb{P}_{\pi}^\Phi$ and $\mathbb{P}_{\pi'}^\Phi$ such that:
\begin{equation*}
    \mathbb{E}_{u,v \sim \Pi}\left[ \|  u-v  \|_p^p  \right] = 0
\end{equation*}
Consequently:
\begin{equation*}
    \mathbb{E}_{u,v \sim \Pi}\left[\sum_{(s,a) \in \mathcal{S} \times \mathcal{A} } |  u_{s,a} -v_{s,a}  |^p  \right] =
  \sum_{(s,a) \in \mathcal{S} \times \mathcal{A} } \mathbb{E}_{u,v \sim \Pi}\left[ |  u_{s,a} -v_{s,a}  |^p  \right] = 0
\end{equation*}
Therefore for all $(s,a) \in \mathcal{S} \times \mathcal{A}$:
\begin{equation*}
    \left| \mathbb{E}_{u \sim \mathbb{P}_\pi^\Phi}\left[ u_{s,a}    \right] - \mathbb{E}_{v \sim \mathbb{P}_{\pi'}^\Phi }\left[   v_{s,a} \right] \right|^p \leq \mathbb{E}_{u,v \sim \Pi}\left[  \left| u_{s,a} - v_{s,a}     \right|^p    \right] = 0
\end{equation*}
Where $u_{s,a}$ and $v_{s,a}$ denote the $(s,a)$ entries of $u$ and $v$ respectively. Notice that for all $(s,a) \in \mathcal{S} \times \mathcal{A}$:
\begin{equation}\label{equation::joint_equality}
    \mathbb{P}_{\pi}^\Phi(s,a) = \mathbb{P}_{\pi'}^\Phi(s,a)
\end{equation}
 Since for all $s \in \mathcal{S}$ and $p \geq 1$:
\begin{equation*}
    \left| \sum_{a \in \mathcal{A}} u_{s,a} - v_{s,a}\right|^p \leq   \sum_{a \in \mathcal{A}} |u_{s,a} - v_{s,a}|^p
\end{equation*}
Therefore for all $s \in \mathcal{S}$:
\begin{equation*}
    \left|\mathbb{E}_{u \sim \mathbb{P}_{\pi}^\Phi} \left[ \sum_{a \in \mathcal{A}} u_{s,a} \right] -\mathbb{E}_{v \sim \mathbb{P}_{\pi'}^\Phi} \left[ \sum_{a \in \mathcal{A}} v_{s,a} \right ]\right|^p \leq \mathbb{E}_{u,v \sim \Pi}\left[   \sum_{a \in \mathcal{A}} |u_{s,a} - v_{s,a}|^p   \right] =0
\end{equation*}
Consequently $\mathbb{P}_{\pi}^\Phi(s) = \mathbb{P}_{\pi'}^\Phi(s)$ for all $s \in \mathcal{S}$. By Bayes rule, this plus equation \ref{equation::joint_equality} yields:
\begin{equation*}
   \mathbb{P}_{\pi}^\Phi(a|s) = \mathbb{P}_{\pi'}^\Phi(a|s)
\end{equation*}
And therefore: $\pi = \pi'$.
\end{proof}

These results can be extended in the following ways:

\begin{enumerate}
    \item In the case of a continuous state space, it is possible to define embeddings using Kernel density estimators. Under the appropriate smoothness conditions on the visitation frequencies, picking an adequate bandwidth and using the appropriate norm to compare different embeddings it is possible to derive similar results to those in Lemma \ref{lemma::policy_equality_simple_appendix} for continuous state spaces.
    \item For embeddings such as $\Phi_5$ in Section \ref{subsection::behavioral_embedding} or $\Phi(\tau)=\sum_{t=0}^H e_{s_t, a_t}$, when $\gamma = 0$, if $\mathrm{WD}_\gamma(\mathbb{P}_{\pi}^\Phi,   \mathbb{P}_{\pi'}^\Phi )  \leq \epsilon$ then $|V(\pi) - V(\pi')|\leq \epsilon R$ for $R = \max_{\tau \in \Gamma} \mathcal{R}(\tau)$ thus implying that a small Wasserstein distance between $\pi$ and $\pi'$s PPEs implies a small difference in their value functions.
\end{enumerate}

\subsection{Random features stochastic gradients}
Let $\phi_\kappa$ and $\phi_\ell$ be two feature maps over $\mathcal{X}$ and $\mathcal{Y}$ and corresponding to kernels $\kappa$ and $\ell$ respectively. For this and the following sections we will make use of the following expression:

    \begin{align}
    G( \mathbf{p}^{\mu}, \mathbf{p}^{\nu}) &= \beta \int_{\mathcal{X}} \left( \mathbf{p}^{\mu}  \right)^\top \phi_\kappa(\mathbf{x}) d \mu(\mathbf{x}, \theta) - \beta \int_{\mathcal{Y}} \left(  \mathbf{p}^{\nu}   \right)^\top \phi_\ell( \mathbf{y} ) d \nu(\mathbf{y})   + \label{equation::random_features_objective}  \\
    &\quad \gamma \beta \int_{\mathcal{X} \times \mathcal{Y}} \exp\left(  \frac{ \left(\mathbf{p}^{\mu }   \right)^\top \phi_\kappa(\mathbf{x}) - \left(  \mathbf{p}^{\nu}  \right)^\top \phi_\ell(\mathbf{y}) - C(\mathbf{x}, \mathbf{y})  }{\gamma}      \right) d \mu(\mathbf{x}) d\nu(\mathbf{y}) \notag
\end{align}

We now show how to compute gradients with respect to the random feature maps:

\begin{lemma}\label{lemma::gradient_random_features}
The gradient $\nabla_{\binom{\mathbf{p}^\mu}{\mathbf{p}^\nu}} G(\mathbf{p}^\mu, \mathbf{p}^\nu)$ of the objective function from Equation \ref{equation::random_features_objective} with respect to the parameters $\binom{\mathbf{p}^{\mu}}{\mathbf{p}^{\nu}}$ satisfies:
\begin{equation*}
    \nabla_{\binom{\mathbf{p}^\mu}{\mathbf{p}^\nu}} G(\mathbf{p}^{\mu}, \mathbf{p}^\nu) = \beta \mathbb{E}_{(\mathbf{x}, \mathbf{y}) \sim \mu \bigotimes \nu }\left[ \left( 1-\exp\left(\frac{( \mathbf{p}^\mu )^\top \phi_\kappa(\mathbf{x}) - (\mathbf{p}^{\nu}   )^\top \phi_\ell - C(\mathbf{x}, \mathbf{y}) }{\gamma}      \right)        \right)\binom{\phi_\kappa(\mathbf{x})}{-\phi_\ell(\mathbf{y} )  } \right]
\end{equation*}

\end{lemma}

\begin{proof}

A simple use of the chain rule, taking the gradients inside the expectation, and the fact that $\mathbf{p}^\mu$ and $\mathbf{p}^\nu$ are vectors yields the desired result. 
\end{proof}

The main consequence of this formulation is the stochastic gradients we use in Algorithm \ref{Alg:wass_constinuous_SGD}.

\subsection{Behavior Guided Policy Gradient and Wasserstein trust region}\label{section::appendix_wasserstein_trust_region}

For a policy $\pi$, we denote as: $V^\pi$, $Q^\pi$ and $A^\pi(s,a) = Q^\pi(s,a) - V^\pi(s)$ the: value function, $Q$-function and advantage function.

The chief goal of this section is to prove Theorem \ref{theorem::policy_improvement_wasserstein_0}. We restate the section's definitions here for the reader's convenience:
To ease the discussion we make the following assumptions:
\begin{itemize}
    \item Finite horizon $T$.
    \item Undiscounted MDP.
    \item States are time indexed. In other words, states visited at time $t$ can't be visited at any other time. 
    \item $\mathcal{S}$ and $\mathcal{A}$ are finite sets. 
\end{itemize}
The third assumption is solely to avoid having to define a time indexed Value function. It can be completely avoided. We chose not to do this in the spirit of notational simplicity. These assumptions can be relaxed, most notably we can show similar results for the discounted and infinite horizon case. We chose to present the finite horizon proof because of the nature of our experimental results. 

Let $\Phi = \mathrm{id}$ be the identity embedding so that $\mathcal{E} = \Gamma$. In this case $\mathbb{P}_{\pi}^\Phi$ denotes the distribution of trajectories corresponding to policy $\pi$. We define the value function $V^\pi : \mathcal{S} \rightarrow \mathbb{R}$ as 
\begin{equation*}
    V^\pi(s_t = s) = \mathbb{E}_{\tau \sim \mathbb{P}_{\pi}^{\mathrm{id}}}\left[ \sum_{\ell = t}^T R(s_{\ell+1}, a_\ell, s_\ell)  | s_t = s \right]
\end{equation*}
The Q-function $Q^\pi: \mathcal{S} \times \mathcal{A} \rightarrow \mathbb{R}$ as:
\begin{equation*}
    Q^\pi(s_t, a_t = a) = \mathbb{E}_{\tau \sim \mathbb{P}_{\pi}^{\mathrm{id}}}\left[    \sum_{\ell = t}^T R(s_{\ell+1}, a_\ell, s_\ell) \right]
\end{equation*}
 Similarly, the advantage function is defined as:
 \begin{equation*}
     A^\pi(s,a) = Q^\pi(s,a) - V^\pi(s)
 \end{equation*}
 We denote by $V(\pi) = \mathbb{E}_{ \tau \sim \mathbb{P}_{\pi}^{\mathrm{id}}}\left[   \sum_{t=0}^T R(s_{t+1}, a_t, s_t)   \right]$ the expected reward of policy $\pi$ and define the visitation frequency as: 
 \begin{equation*}
 \rho_{\pi}(s) = \mathbb{E}_{\tau \sim \mathbb{P}_{\pi}^{\mathrm{id}} } \left[    \sum_{t=0}^T \mathbf{1}(s_t = s)     \right]
\end{equation*}
The first observation in this section is the following lemma:
\begin{lemma}
two distinct policies $\pi$ and $\tilde{\pi}$ can be related via the following equation :
\begin{equation*}
    V(\tilde{\pi}) = V(\pi)  + \sum_{s \in \mathcal{S}}\left( \rho_{\tilde{\pi}}(s) \left( \sum_{ a \in \mathcal{A}} \tilde{\pi}(a | s) A^\pi(s,a)  \right) \right)
\end{equation*}
\end{lemma}
\begin{proof}
Notice that $A^\pi(s,a) = \mathbb{E}_{s' \sim P(s'| a,s) } \left[  R(s', a, s) + V^\pi(s') - V^\pi(s)     \right]$. Therefore:
\begin{align*}
    \mathbb{E}_{\tau \sim \mathbb{P}_{\tilde{\pi}}^{\mathrm{id}}}\left[ \sum_{t=0}^T A_\pi(s_t, a_t)      \right] &= \mathbb{E}_{\tau \sim \mathbb{P}_{\tilde{\pi}}^{\mathrm{id}}}\left[ \sum_{t=0}^T   R(s_{t+1}, a_t, s_t) + V^\pi(s_{t+1}) - V^\pi(s_t)    \right] \\
    &= \mathbb{E}_{\tau \sim \mathbb{P}_{\tilde{\pi}}^{\mathrm{id}}}\left[ \sum_{t=0}^T   R(s_{t+1}, a_t, s_t)\right]  - \mathbb{E}_{s_0}\left[  V^\pi(s_0)\right]     \\
    &= -V(\pi) + V(\tilde{\pi})
\end{align*}
The result follows.
\end{proof}
See \cite{sutton1998introduction} for an alternative proof. We also consider the following linear approximation to $V$ around policy $\pi$ (see: \cite{kakade2002approximately}):
\begin{equation*}
    L(\tilde{\pi}) = V(\pi) +   \sum_{s\in \mathcal{S}}\left( \rho_{\pi}(s) \left( \sum_{a \in \mathcal{A}} \tilde{\pi}(a | s) A^\pi(s,a)  \right) \right)
\end{equation*}

Where the only difference is that $\rho_{\tilde{\pi}}$ was substituted by $\rho_{\pi}$. Consider the following embedding $\Phi^s : \Gamma \rightarrow \mathbb{R}^{|\mathcal{S}|}$ defined by $\left( \Phi(\tau ) \right)_s=    \sum_{t=0}^T \mathbf{1}(s_t = s)$, and related cost function defined as: $C(\mathbf{v},\mathbf{w}) = \|\mathbf{v}-\mathbf{w}\|_1$. 

\begin{lemma}
The Wasserstein distance $\mathrm{WD}_0( \mathbb{P}_{\tilde{\pi}}^{\Phi^s}, \mathbb{P}_{\pi}^{\Phi^s})$  is related to visit frequencies since:  
\begin{equation*}\mathrm{WD}_0(  \mathbb{P}_{\tilde{\pi}}^{\Phi^s}, \mathbb{P}_{\pi}^{\Phi^s})\geq  \sum_{s \in \mathcal{S}} |  \rho_\pi(s) - \rho_{\tilde{\pi}}(s) | 
\end{equation*}
\end{lemma}
\begin{proof}
Let $\Pi$ be the optimal coupling between $\mathbb{P}_{\tilde{\pi}}^{\Phi^s}$ and $\mathbb{P}_{\pi}^{\Phi^s}$. Then:

\begin{align*}
    \mathrm{WD}_0( \mathbb{P}_{\tilde{\pi}}^{\Phi^s}, \mathbb{P}_{\pi}^{\Phi^s}) &= \mathbb{E}_{u,v \sim \Pi}\left[  \| u - v\|_1 \right]\\
    &= \sum_{s\in\mathcal{S}} \mathbb{E}_{u,v \sim \Pi}\left[  |u_s - v_s|\right]
 \end{align*}
Where $u_s$ and $v_s$ denote the $s \in \mathcal{S}$ indexed entry of the $u$ and $v$ vectors respectively.
Notice that for all $s \in \mathcal{S}$ the following is true:
\begin{equation*}
    \left| \underbrace{ \mathbb{E}_{u \sim  \mathbb{P}_{\pi}^{\Phi^s}}\left[ u_s\right] }_{\rho_\pi(s)} -  \underbrace{\mathbb{E}_{v \sim  \mathbb{P}_{\pi}^{\Phi^s}}\left[v_s   \right]}_{\rho_{\pi'}(s)}      \right| \leq  \mathbb{E}_{u,v \sim \Pi}\left[  |u_s - v_s|\right]
\end{equation*}
The result follows.

\end{proof}

 These observations enable us to prove an analogue of Theorem 1 from \cite{trpo}, namely:

\begin{theorem}\label{theorem::policy_improvement_wasserstein}
If $ \mathrm{WD}_0(  \mathbb{P}_{\tilde{\pi}}^{\Phi^s}, \mathbb{P}_{\pi}^{\Phi^s}) \leq \delta$ and $\epsilon  = \max_{s,a} | A^\pi(s,a)  |$, then $V(\tilde{\pi}) \geq L(\tilde{\theta}) - \delta \epsilon$.  
\end{theorem}
As in \cite{trpo}, Theorem \ref{theorem::policy_improvement_wasserstein_0} implies a policy improvement guarantee for BGPG from Section \ref{sec:bgpg}.

\begin{proof}
Notice that:
\begin{align*}
    V(\tilde{\pi}) - L (\tilde{\pi}) =   \sum_{s \in \mathcal{S} } \left( \left( \rho_{\tilde{\pi}}(s) - \rho_{\pi}(s) \right) \left( \sum_{a \in \mathcal{A}} \tilde{\pi}(a | s) A^\pi(s,a)   \right) \right)
\end{align*}
Therefore by Holder inequality:
\begin{equation*}
   | V(\tilde{\pi}) - L(\tilde{\pi}) | \leq \underbrace{ \left(  \sum_{s \in \mathcal{S}} |  \rho_\pi(s) - \rho_{\tilde{\pi}}(s) |  \right) }_{  \leq \mathrm{WD}_0(  \mathbb{P}_{\tilde{\pi}}^{\Phi^s}, \mathbb{P}_{\pi}^{\Phi^s}) \leq \delta } \underbrace{ \left( \sup_{s \in \mathcal{S}}  \left|  \sum_{a \in \mathcal{A}} \tilde{\pi}(a | s) A^\pi(s,a)      \right| \right) }_{\leq \epsilon }
\end{equation*}
The result follows.
\end{proof}
We can leverage the results of Theorem \ref{theorem::policy_improvement_wasserstein} to show wasserstein trust regions methods with embedding $\Phi^s$ give a monotonically improving sequence of policies. The proof can be concluded by following the logic of Section 3 in \cite{trpo}.

\subsection{Off policy embeddings and their properties.}

It is easy to see that if the cost function equals the $l_2$ norm between state-policy pairs and if $\mathrm{WD}_0( \mathbb{P}_\pi^{\Phi_\mathcal{S}}, \mathbb{P}_{\pi'}^{\Phi_\mathcal{S}} ) = 0$ then $\mathbb{E}_{\mathbb{P}_\mathcal{S}}\left[ \mathrm{1}(\pi(S) \neq \pi'(S))\right] = 0$. If $\mathbb{P}_\mathcal{S}$ has mass only in relevant areas of the state space, a value of zero implies the two policies behave similarly where it matters. In the case when the user may care only about the action of a policy within a set of states of interest, this notion applies. 

When the sampling distribution can be identified with the stationary distribution over states of the current policy, we can recover trust region-type of results for BGPG.


\end{document}